\documentclass{article}
\usepackage{iclr2027_conference,times}
\usepackage[T1]{fontenc}
\usepackage[utf8]{inputenc}
\usepackage{amsmath,amssymb,mathtools}
\usepackage{amsthm}
\newtheorem{theorem}{Theorem}
\newtheorem{proposition}{Proposition}
\usepackage{graphicx,booktabs,longtable,array,calc}
\usepackage{caption}
\usepackage[section]{placeins}
\usepackage{float}
\usepackage{fvextra}
\usepackage{microtype}
\usepackage{xurl}
\usepackage[unicode,hidelinks]{hyperref}
\hypersetup{pdftitle={Cortex: A Fixed-Point Theory of Governed Coding Agents},pdfauthor={Marius-Constantin Dinu and Florian Zeba},
  pdfsubject={Conditional semantics and evaluation methodology for governed coding agents}}
\DeclareUnicodeCharacter{03C4}{\ensuremath{\tau}}
\DeclareUnicodeCharacter{2032}{\ensuremath{\prime}}
\DeclareUnicodeCharacter{2227}{\ensuremath{\land}}
\DeclareUnicodeCharacter{2192}{\ensuremath{\rightarrow}}
\providecommand{\tightlist}{\setlength{\itemsep}{0pt}\setlength{\parskip}{0pt}}
\RecustomVerbatimEnvironment{verbatim}{Verbatim}{breaklines=true,fontsize=\footnotesize}
\newsavebox{\pandocbox}
\newcommand{\pandocbounded}[1]{%
  \sbox{\pandocbox}{#1}%
  \ifdim\wd\pandocbox>\linewidth
    \resizebox{\linewidth}{!}{\usebox{\pandocbox}}%
  \else\usebox{\pandocbox}\fi}
\title{Cortex: A Fixed-Point Theory of Governed Coding Agents\thanks{Canonical paper and latest results:\newline
\url{https://benchmark.cortex.a2olabs.com}\newline
\url{https://cortex.a2olabs.com}\newline
Benchmark source: \url{https://github.com/Xpitfire/cortex-gauntlet}.}}
\author{Marius-Constantin Dinu \& Florian Zeba\\ Alpha Omega Labs}
\iclrfinalcopy
\begin{document}
\maketitle
% The official final mode reveals authors; replace its conference-acceptance header.
\lhead{Preprint. September 15, 2026}
\begin{abstract}

A \emph{coding agent} combines a large language model (LLM) with a harness that plans, edits, and executes code. We study \textbf{Cortex}, a supervisory layer for pre-execution governance, requirement tracking, and iterative validation and repair. We formalize requirement closure on a finite lattice. For a fixed requirement set and a sound, monotone validation rule with persistent evidence, exhaustive iteration reaches the least fixed point with no more strict increases than there are requirements. Fair consequence scheduling reaches the same closure. A separate stochastic model bounds expected completion time under a uniform positive-progress assumption. These results depend on their stated assumptions; artifact repair, fallible validation, and finite execution budgets do not satisfy them automatically.

We define Gauntlet's trajectory-similarity and build-gated scores, attack-success estimators, Wilson intervals, and case-cluster bootstrap summaries. Effective-feedback accounting counts newly certified requirements, rather than estimating their information value or utility. The five evaluation families support raw-versus-governed comparisons. We report retained benchmark observations, including rescoring of existing artifacts; constructed outcomes are excluded. Unequal observation sets and incomplete provenance prevent causal claims about governance or measured long-horizon scaling.

\end{abstract}

\section{1. Introduction}\label{introduction}

An LLM coding harness turns model output into tool calls, file changes, and shell commands. Task competence and execution control are distinct properties: a capable agent can follow malicious instructions \citep{ref-46,ref-20}, while a restricted agent can still leave requirements unresolved on a long task \citep{ref-21}.

\textbf{Cortex} configures governance and orchestration around a base harness. Governance applies pre-execution checks and capability/dependency policies. Orchestration analyzes instructions, tracks requirements, and coordinates validation and repair within a resource budget. Action filtering alone does not establish planning competence or complete request-level safety. Validation may use executable checks, fixed rules, or model-based judgments; their reliability is a separate evaluation question. We study these contracts through the \textbf{Gauntlet} benchmark framework.

Inference-time sampling, search, and revision trade additional computation for task performance \citep{ref-26,ref-27}, complementing training-time scaling laws \citep{ref-24,ref-25}. The effective-feedback perspective \citep{ref-21} motivates tracking valid, nonredundant feedback that remains useful after later edits, rather than counting only tokens or tool calls. Our requirement count is an operational measure of retained progress. Classical value of information \citep{ref-22} and value of computation \citep{ref-23} also require beliefs, utilities, and costs; a ledger update does not determine either value.

We address \emph{(i) what an idealized supervisory layer guarantees under stated assumptions} and \emph{(ii) how its implemented behavior can be evaluated}. The least fixed point in (i) is the closure of a specified validation rule, which may leave required work unresolved. The metrics and comparison protocol in (ii) do not establish that a particular implementation satisfies the theorem's hypotheses or that historical arm differences isolate governance alone.

\subsection{Contributions}\label{contributions}

\begin{itemize}
\tightlist
\item
  A typed action-filter model and a conditional finite-lattice model of requirement closure (Sections 3--4).
\item
  A proof of finite closure, oracle-relative soundness, and fair-schedule independence under explicit assumptions (Theorem 1), plus accounting and conditional hitting-time bounds (Propositions 1--2).
\item
  Definitions and domain restrictions for the implemented evaluation metrics (Section 5).
\item
  A five-family evaluation protocol, with implementation and historical-evidence limitations made explicit rather than treated as experimental conclusions (Sections 6--9).
\end{itemize}

\section{2. Related Work and Theoretical Lineage}\label{related-work-and-theoretical-lineage}

\textbf{Neurosymbolic and cognitive computation.} Combining neural generation with symbolic constraints is one form of neurosymbolic integration \citep{ref-6,ref-7}. Validation--repair cycles also resemble the repeated rule application of classical cognitive architectures: the problem-space account of Newell and Simon \citep{ref-8}, SOAR \citep{ref-9}, and ACT-R \citep{ref-10}. Section 4 isolates a finite consequence operator, rather than modeling a complete cognitive architecture.

\textbf{Fixed-point semantics.} Our convergence argument rests on the least-fixed-point semantics of monotone operators: van Emden \& Kowalski's immediate-consequence operator and its least fixed point as the meaning of a logic program \citep{ref-2}, the Knaster--Tarski lattice fixed-point theorem \citep{ref-3}, and Kleene iteration \citep{ref-4}. Where an operator is a contraction on a metric space, Banach's theorem \citep{ref-5} gives uniqueness and a rate; we use the lattice formulation because the requirement state space is naturally a finite complete lattice.

\textbf{Evaluation.} Capability over trajectories is scored by VERTEX, defined in Section 5.1: embedding cross-similarity against a reference descriptor set, with a presence component (BERTScore-style bidirectional matching \citep{ref-11}) and an order component (distance-decayed dynamic time warping \citep{ref-12}), each affinely recalibrated against the matrix's mean-similarity baseline. Embedding comparison also appears in trajectory evaluation \citep{ref-1,ref-19}. Gauntlet applies this construction to extracted descriptors, whose order need not represent an observed execution trajectory. For integers \(1\le k\le n\), the any-seed attack-success estimate uses \(\mathrm{pass}@k\) \citep{ref-15}. Under independent trials with common success probability \(p\in[0,1]\), it is unbiased for \(1-(1-p)^k\). Success rates carry Wilson intervals \citep{ref-13} and bootstrap summaries \citep{ref-14}. The optional rubric judge draws on G-Eval \citep{ref-16}, MT-Bench \citep{ref-17}, and Prometheus \citep{ref-18}. Their bias analyses motivate judge validation but do not establish the validity of Gauntlet's judges. Safety draws on the jailbreak \citep{ref-46} and real-world prompt-injection \citep{ref-20} literatures.

\textbf{Scaling, test-time compute, and effective feedback.} Training scaling laws relate compute, data, and loss \citep{ref-24,ref-25}; inference-time search and revision trade compute for accuracy \citep{ref-26,ref-27}. Effective-feedback work \citep{ref-21} motivates measuring usable feedback. We count newly certified requirements, as formalized in Section 4.3. This finite-set accounting identity neither derives an empirical scaling law nor establishes that equal counts represent equal information or utility.

\textbf{Iterative feedback, verification, and metareasoning.} ReAct interleaves reasoning and interaction \citep{ref-31}; Self-Refine and Reflexion study feedback and revision \citep{ref-29,ref-30}. Process supervision is a distinct approach to evaluating intermediate reasoning steps \citep{ref-28}. Repository issue resolution is studied in benchmarks such as SWE-bench \citep{ref-32}; it is not the common evaluation setting of all these methods. Our theorem concerns any operator satisfying its assumptions. An implementation requires a separate conformance argument. Value of information \citep{ref-22} and rational metareasoning \citep{ref-23} provide decision-theoretic context, without supplying a utility model or an optimal computation-selection policy for the present framework.

\begin{figure}
\centering
\includegraphics[width=1\linewidth,height=\textheight,keepaspectratio]{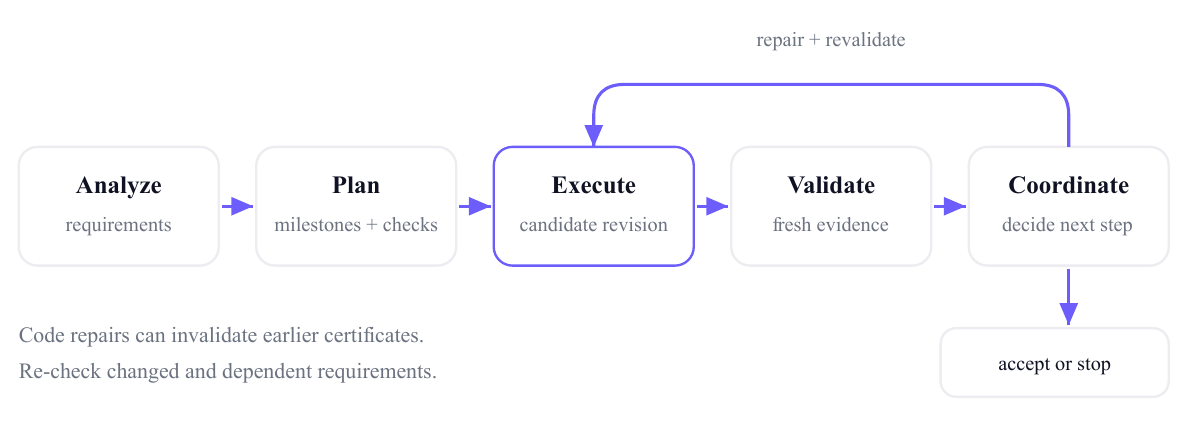}
\caption{Schematic execution--validation--repair architecture. Its consequence-operator abstraction requires a sufficient requirement state, monotone validation, and sound persistent certificates (Section 4.2). The operational repair cycle does not establish those assumptions.}\label{fig-1}
\end{figure}

\section{3. Preliminaries and Notation}\label{preliminaries-and-notation}

We write \(\mathbb{N}_0=\{0,1,\ldots\}\), \(\mathbb{N}_{+}=\{1,2,\ldots\}\), \(\mathbb{R}\) for the reals, and \(\mathbb{I}=[0,1]\). For finite \(X\), \(|X|\) denotes cardinality and \(2^X\) the power set. For nonempty finite or countable \(X\), \(\Delta(X)=\{p\in\mathbb{R}_{\ge0}^{X}:\sum_{x\in X}p_x=1\}\) is the set of discrete probability distributions. For \(d\in\mathbb{N}_{+}\), \(\mathbb{S}^{d-1}=\{v\in\mathbb{R}^d:\lVert v\rVert_2=1\}\), with Euclidean inner product \(\langle u,v\rangle\in[-1,1]\) for \(u,v\in\mathbb{S}^{d-1}\). For real \(x,a,b\) with \(a\le b\), \(\operatorname{clamp}(x;a,b)=\min(\max(x,a),b)\). The indicator \(\mathbf{1}\{E\}\) equals \(1\) when statement \(E\) holds and \(0\) otherwise.

\textbf{Order-theoretic objects.} A complete lattice \((L,\sqsubseteq)\) has a least upper bound and a greatest lower bound for every subset, including bottom and top elements. A map \(f:L\to L\) is \emph{monotone} when \(x\sqsubseteq y\) implies \(f(x)\sqsubseteq f(y)\), and \emph{inflationary} when \(x\sqsubseteq f(x)\). These are distinct properties. A fixed point satisfies \(f(x)=x\). A monotone map on a complete lattice has a least fixed point \(\operatorname{lfp}(f)\) \citep{ref-3}.

\textbf{Agent objects.} Let \(\mathcal{A}\) and \(\mathcal{O}\) be nonempty countable action and observation spaces. A distinguished outcome \(\bot_{\!a}\notin\mathcal{A}\) denotes refusal without execution, not the lattice bottom. Set \(\mathcal{A}_{\bot}=\mathcal{A}\cup\{\bot_{\!a}\}\) and include refusal observations in \(\mathcal{O}\). The history space \(\mathcal{H}=(\mathcal{A}_{\bot}\times\mathcal{O})^*\) contains all finite execution histories, including the empty history. A base policy \(\pi:\mathcal{H}\to\Delta(\mathcal{A})\) proposes ordinary actions. For a finite, possibly empty index set \(J\), let \(h_j:\mathcal{H}\times\mathcal{A}\to\{0,1\}\), \(j\in J\), be hooks; \(1\) means blocked.

\textbf{Task objects.} Fix a finite requirement set \(R=\{r_1,\ldots,r_N\}\), \(N\in\mathbb{N}_0\), and required subset \(R_{\mathrm{req}}\subseteq R\). The abstract state \(S\subseteq R\) records certified requirements. The lattice \((2^R,\subseteq)\) has join \(\cup\), meet \(\cap\), bottom \(\varnothing\), and top \(R\). The abstraction \(\nu:2^R\times R\to\{0,1\}\) determines which further certificates can be derived from \(S\). Representing evidence by requirement identifiers alone requires the sufficiency assumption in Section 4.2; an actual workspace, tests, and execution history contain more information.

\textbf{Metric objects.} Let \(\mathcal{D}\) be finite text descriptors and \(\mathcal{D}^{*}=\bigcup_{m\in\mathbb{N}_0}\mathcal{D}^{m}\) their finite sequences. The embedding \(\varphi\) returns a unit vector for nonzero embeddings; the implementation uses the zero vector for a descriptor with no embedding features. Thus its full codomain is \(\mathbb{S}^{d-1}\cup\{0\}\), and an exact text match need not score \(1\) for a zero vector. Signals \(s_k\in\mathbb{I}\) combine with weights \(w\in\Delta(\{1,\ldots,K\})\), \(K\in\mathbb{N}_{+}\).

\section{4. The Control Layer: Governance and Orchestration}\label{the-control-layer-governance-and-orchestration}

The action filter and requirement-closure operator describe different contracts. The first restricts proposed actions at a specified history; the second describes an idealized validation closure. Neither alone proves real-world task correctness.

\subsection{4.1 Governance as a projection on the action space}\label{governance-as-a-projection-on-the-action-space}

At fixed \(H\), let \(\mathcal{A}_G(H)=\{a\in\mathcal{A}:h_j(H,a)=0\text{ for every }j\in J\}\). Define \(G_H:\mathcal{A}_{\bot}\to\mathcal{A}_{\bot}\) by

\[ G_H(a)=\begin{cases}
 a,&a\in\mathcal{A}_G(H),\\
 \bot_{\!a},&a\notin\mathcal{A}_G(H).
 \end{cases} \]

In particular \(G_H(\bot_{\!a})=\bot_{\!a}\). Therefore \(G_H\circ G_H=G_H\): it is an idempotent retraction onto \(\mathcal{A}_G(H)\cup\{\bot_{\!a}\}\). The governed policy is the pushforward of the proposal policy:

\[ \pi_G(a\mid H)=\pi(a\mid H)\mathbf{1}\{a\in\mathcal{A}_G(H)\}
 \quad(a\in\mathcal{A}),\qquad
 \pi_G(\bot_{\!a}\mid H)=\sum_{a\in\mathcal{A}\setminus\mathcal{A}_G(H)}\pi(a\mid H). \]

Blocked probability becomes refusal mass, \textbf{not renormalized admissible mass}. Conditional resampling until an admissible action is obtained would be a different policy and is undefined if the original policy assigns zero mass to admissible actions. Filtering leaves accepted actions unchanged at this history, but can change later behavior and task utility.

\textbf{Assumption A1 (complete mediation and hook coverage).} Every effect in the claimed protection boundary is checked before execution against authoritative context, without a check/use race; every unsafe action in the specified threat class is blocked by at least one hook. Under A1 no action in that class executes. Pure predicates alone do not establish coverage, and neither false-positive rates nor semantic safety follow from idempotence. The benchmark must measure refusals and coverage separately.

\subsection{4.2 Orchestration: the completion loop as a consequence operator}\label{orchestration-the-completion-loop-as-a-consequence-operator}

For a \textbf{fixed} oracle and requirement set define

\[ \Phi(S)=S\cup\{r\in R:\nu(S,r)=1\}. \]

\textbf{Assumption A2 (oracle and evidence).} Requirement identifiers are a sufficient summary for this deterministic rule. Validation is monotone: \(S\subseteq T,\ \nu(S,r)=1\Rightarrow\nu(T,r)=1\). Whenever the certificates in \(S\) are valid for the specified task, every new certificate \(r\in R\setminus S\) with \(\nu(S,r)=1\) must be supported by valid evidence and correct for that task. Accepted transitions must preserve previously valid certificates. The union makes \(\Phi\) inflationary; the additional monotonicity assumption makes it monotone. Append-only storage alone does \textbf{not} imply that assumption. For example, \(R=\{a,b\}\), \(\nu(S,a)=0\) for all \(S\), and \(\nu(S,b)=\mathbf{1}\{a\notin S\}\) give \(\Phi(\varnothing)=\{b\}\) but \(\Phi(\{a\})=\{a\}\).

\begin{theorem}[finite closure, conditional soundness, fair-schedule independence]\label{thm-1}

Let \(\Phi:2^R\to2^R\) be monotone and inflationary. Starting at \(S_0=\varnothing\), the iterates \(S_{t+1}=\Phi(S_t)\) reach \(S_*=\operatorname{lfp}(\Phi)\) after at most \(|R|\) strict increases. A further evaluation may be needed to detect stabilization. With A2's soundness and persistence conditions, all certificates in \(S_*\) remain valid. Define PASS at closure by \(R_{\mathrm{req}}\subseteq S_*\); PASS then certifies those requirements relative to the oracle, not all possible correctness properties. Starting at \(\varnothing\), adding enabled consequences in any \textbf{fair} schedule reaches the same \(S_*\).

\end{theorem}

\begin{proof}

Inflationarity gives an increasing chain. Each strict increase adds at least one of the \(|R|\) elements, so the chain stabilizes. If \(F\) is any fixed point, then \(S_0\subseteq F\) and monotonicity implies \(S_{t+1}=\Phi(S_t)\subseteq\Phi(F)=F\) by induction. The stabilized state is thus the least fixed point. Soundness follows by induction on certified additions, using both oracle soundness and persistence. For asynchronous addition, every added requirement is in \(\Phi(S)\), so the same induction keeps the state below \(S_*\). Fairness means no persistently enabled missing requirement is postponed forever. Since there can be only finitely many additions, the eventual state has no enabled missing requirement and is a fixed point; leastness makes it \(S_*\). \qedhere

\end{proof}

The theorem proves derivability under \(\nu\), \textbf{not} that \(S_*\) is exactly the set of requirements achievable by arbitrary agents. Completeness of the oracle/generator would be an additional assumption. An unfair scheduler can omit enabled work forever. Detecting one unproductive \emph{stochastic} attempt is not detecting a fixed point of a deterministic exhaustive consequence operator.

\textbf{Implementation boundary.} Real repair passes edit a shared artifact and can regress previously passing checks. A historical union of successful checks is not a current correctness certificate. The benchmark uses bounded retries, observation-dependent validation, and score-based candidate selection; these are not an implementation proof of A2 or exhaustive closure. Fresh validation of the final artifact is needed for PASS. Raw coding harnesses also take multiple internal actions; a single benchmark invocation is not a single consequence or a single requirement addition.

\begin{figure}
\centering
\includegraphics[width=1\linewidth,height=\textheight,keepaspectratio]{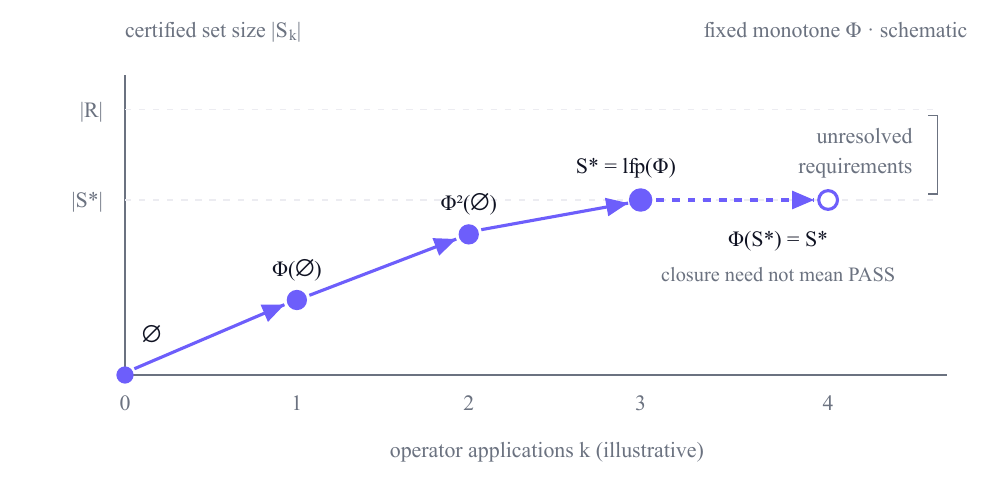}
\caption{Illustrative Kleene ascent for the fixed monotone operator of Theorem 1, not a measured agent trace. At most \(|R|\) strict increases occur; one further evaluation may detect closure. A fair schedule reaches the same least fixed point, which can leave required work unresolved.}\label{fig-2}
\end{figure}

\subsection{4.3 Effective feedback and the value of computation}\label{effective-feedback-and-the-value-of-computation}

Define the count \(\mathrm{efc}(S)=|\Phi(S)\setminus S|\) for the abstract operator. It measures newly certified requirements, motivated by retained-feedback accounts \citep{ref-21}. Along iteration from \(\varnothing\) under A2, these additions are valid, nonredundant, and retained. Their information content and utility are not necessarily equal.

\begin{proposition}[effective-feedback accounting]\label{prop-1}

If \(S_0=\varnothing\) and \(S_T=S_*=\operatorname{lfp}(\Phi)\) is the stabilized iterate, then \(\sum_{t=0}^{T-1}|S_{t+1}\setminus S_t|=|S_*|\). If \(R_{\mathrm{req}}\subseteq S_*\), exactly \(|R_{\mathrm{req}}|\) of these additions are required certificates. The \textbf{total} number of additions before first completion can exceed \(|R_{\mathrm{req}}|\) because optional requirements can be added too.

\end{proposition}

\begin{proof}

The disjoint differences \(S_{t+1}\setminus S_t\) partition \(S_*\). Intersecting that partition with \(R_{\mathrm{req}}\) gives the required-only identity. \qedhere

\end{proof}

One productive pass can add any number from \(1\) to \(|R\setminus S|\). For example, \(\Phi(S)=R\) completes arbitrarily many requirements in one pass. Consequently no lower bound on required passes, no impossibility for single-invocation harnesses, and no advantage over independent sampling follows from this accounting. An append-only ledger prevents duplicate \textbf{credit}, not duplicate computation. Positive decision-theoretic value of information or net value of computation requires an explicit utility/cost model \citep{ref-22,ref-23}; \(\mathrm{efc}>0\) is not equivalent to either.

\subsection{4.4 Probabilistic execution over the fixed-point lattice}\label{probabilistic-execution-over-the-fixed-point-lattice}

A stochastic process is separate from the deterministic closure model. Let \(\nu_{\mathrm{obs}}:\mathcal{H}\times2^R\times\mathcal{A}_{\bot}\times
\mathcal{O}\times R\to\{0,1\}\) validate observed executions. Define \(U(H,S,a,o)=S\cup\{r\in R:\nu_{\mathrm{obs}}(H,S,a,o,r)=1\}\) as historical certificate accumulation. A distinct current-artifact validation set may shrink. For a requirement-only, time-homogeneous Markov abstraction, \textbf{assume} the conditional law of the next requirement state depends on the past only through \(S\). Then

\[ P(S,S')=\sum_{\substack{a\in\mathcal{A}_{\bot},\,o\in\mathcal{O}\\U(S,a,o)=S'}}
 \pi_G(a\mid S)\,E(o\mid S,a),\qquad \sum_{S'\subseteq R}P(S,S')=1, \]

where \(\pi_G(\cdot\mid S)\in\Delta(\mathcal{A}_{\bot})\) and \(E(\cdot\mid S,a)\in\Delta(\mathcal{O})\) is the observation distribution, including a refusal observation for \(\bot_{\!a}\). Each of \(U,\pi_G,E\) must admit this sufficient-state representation. For general history-dependent execution use \(X_t=(H_t,S_t)\) instead, including all relevant environment state in the history. Aggregating its transitions by \(S\) alone does not in general produce a Markov chain. Historical certificate accumulation gives \(P(S,S')=0\) unless \(S\subseteq S'\) but does not prove current-artifact correctness.

\textbf{Completion as a hitting time.} Let \(C=\{S\subseteq R:R_{\mathrm{req}}\subseteq S\}\) and \(T_C=\inf\{t\in\mathbb{N}_0:S_t\in C\}\), with \(\inf\varnothing=\infty\). For an increasing process \(C\) is a closed set; make its states absorbing if the run stops at first completion. Specify the initial state \(S_0\) for all probabilities and expectations.

\begin{proposition}[conditional completion time]\label{prop-2}

For a finite increasing Markov chain, suppose every reachable \(S\notin C\) has \(\Pr(S_{t+1}\supsetneq S_t\mid S_t=S)\ge\varepsilon\), where \(0<\varepsilon\le1\). Then \(\Pr(T_C<\infty)=1\) and \(\mathbb{E}[T_C]\le (|R|-|S_0|)/\varepsilon\).

\end{proposition}

\begin{proof}

Since \(R\in C\), at most \(|R|-|S_0|\) strict increases can precede completion. Before completion, the wait for each increase is dominated by a geometric variable of mean \(1/\varepsilon\). Their sum bounds \(T_C\) and its expectation; independence of the holding times is unnecessary. A finite expectation implies almost-sure completion. \qedhere

\end{proof}

An incomplete absorbing state violates the progress assumption. Without that assumption, completion need not occur and \(\mathbb{E}[T_C]\) may be infinite. Governance changes the execution law through the refusal-mass filter of Section 4.1; comparing unsafe-event probabilities additionally requires an explicit unsafe-event indicator on the augmented state, not merely a subset of satisfied requirements. A difference of expected completion times is meaningful only when both expectations are finite. No transition probabilities or completion-time estimates are supplied by the lattice theorem.

\section{5. Evaluation}\label{evaluation}

\subsection{5.1 VERTEX trajectory similarity}\label{vertex-trajectory-similarity}

This implementation of VERTEX combines descriptor cross-similarity with an order penalty, inspired by trajectory evaluation \citep{ref-1,ref-19}. For nonempty sequences \(c=(c_1,\ldots,c_m)\) and \(r=(r_1,\ldots,r_n)\), \(m,n\in\mathbb{N}_{+}\), define \(\mathbf{S}_{ij}=\langle\varphi(c_i),\varphi(r_j)\rangle\in[-1,1]\). For unit embeddings this is cosine similarity; zero-vector handling is as in Section 3. Either empty sequence receives score \(0\) by convention. Figure 3 shows the two components.

\emph{Presence (unordered).} A BERTScore-style bidirectional best match \citep{ref-11}, where \(P\) (precision) measures how well each candidate descriptor is matched by some reference and \(Q\) (recall) how well each reference descriptor is matched by some candidate:

\[\begin{gathered} P=\frac{1}{m}\sum_{i=1}^{m}\max(0,\max_j\mathbf{S}_{ij}),\\
 Q=\frac{1}{n}\sum_{j=1}^{n}\max(0,\max_i\mathbf{S}_{ij}),\\
 F=\begin{cases}2PQ/(P+Q),&P+Q>0,\\0,&P+Q=0.\end{cases} \end{gathered}\]

Thus \(P,Q,F\in\mathbb{I}\) and \(\min(P,Q)\le F\le\max(P,Q)\). Negative similarities still enter the order cost and the matrix-mean baseline below.

\emph{Order.} Dynamic time warping \citep{ref-12} minimizes over paths \(\Gamma\) from \((1,1)\) to \((m,n)\) with steps \((1,0)\), \((0,1)\), or \((1,1)\). For finite \(\lambda\ge0\), the cost is \(d(i,j)=(1-\mathbf{S}_{ij})(1+\lambda|(i-1)/m-(j-1)/n|)\):

\[ \mathrm{DTW}=\min_{\gamma\in\Gamma}\sum_{(i,j)\in\gamma}d(i,j),\qquad
   D=\max\!\Bigl(0,\,1-\tfrac{\mathrm{DTW}}{m+n}\Bigr)\in\mathbb{I}, \]

The path has between \(\max(m,n)\) and \(m+n-1\) entries. The denominator \(m+n\) is a chosen normalization, not the path length or an upper bound on the \textbf{cost}.

\emph{Matrix-mean recalibration.} For \(0<\epsilon<1\), let \(b=\operatorname{clamp}((mn)^{-1}\sum_{i,j}\mathbf{S}_{ij};0,1-\epsilon)\) and \(\eta_b(x)=\operatorname{clamp}((x-b)/(1-b);0,1)\). The raw mean lies in \([-1,1]\), and the clamp ensures a positive denominator. It is the mean of this candidate--reference pair, \textbf{not} an independently estimated null expectation. For finite \(\alpha\in[0,1]\):

\[ \mathrm{VERTEX}(c,r)\;=\;\alpha\,\eta_b(F)+(1-\alpha)\,\eta_b(D)\;\in\;\mathbb{I}. \]

The benchmark configuration fixes \(\alpha=0.6\), \(\lambda=0.5\), and \(\epsilon=0.01\) (the baseline-clamp margin). Proposition 3 records the range and the exact sense of the calibration.

\begin{proposition}[range and anchors]\label{prop-3}

VERTEX lies in \(\mathbb{I}\). The normalizer is nondecreasing on \(\mathbb{R}\) and strictly increasing on \([b,1]\), with \(\eta_b(b)=0\) and \(\eta_b(1)=1\). Identical sequences of nonzero unit embeddings score \(1\) in exact arithmetic.

\end{proposition}

\begin{proof}

All floored best matches lie in \(\mathbb{I}\); the harmonic mean lies between its arguments. Costs are nonnegative, hence the clamped \(D\) lies in \(\mathbb{I}\). The normalizer has the stated endpoints and slope \(1/(1-b)>0\) on its unclamped interval. The mixture is convex. Identical unit-vector sequences have unit diagonal similarities, \(P=Q=1\), and a zero-cost diagonal path, so \(F=D=1\). \qedhere

\end{proof}

There is \textbf{no universal chance floor}. A constant zero similarity matrix with \(m=n\) has \(F=0,D=1/2\), hence score \(0.2\) at the default mixture. With \(m=1,n=9\) and \(\lambda=0.5\), the same zero similarities give path cost \(11\), \(D=0\), and score \(0\). Lengths, order, and embedding geometry therefore affect unrelated-input scores. Maxima need not strictly exceed a mean, and the DTW term need not concentrate near \((1+b)/2\). Comparisons require fixed extraction and embedding protocols; they are similarity comparisons, not calibrated probabilities or correctness certificates.

\begin{figure}
\centering
\includegraphics[width=1\linewidth,height=\textheight,keepaspectratio]{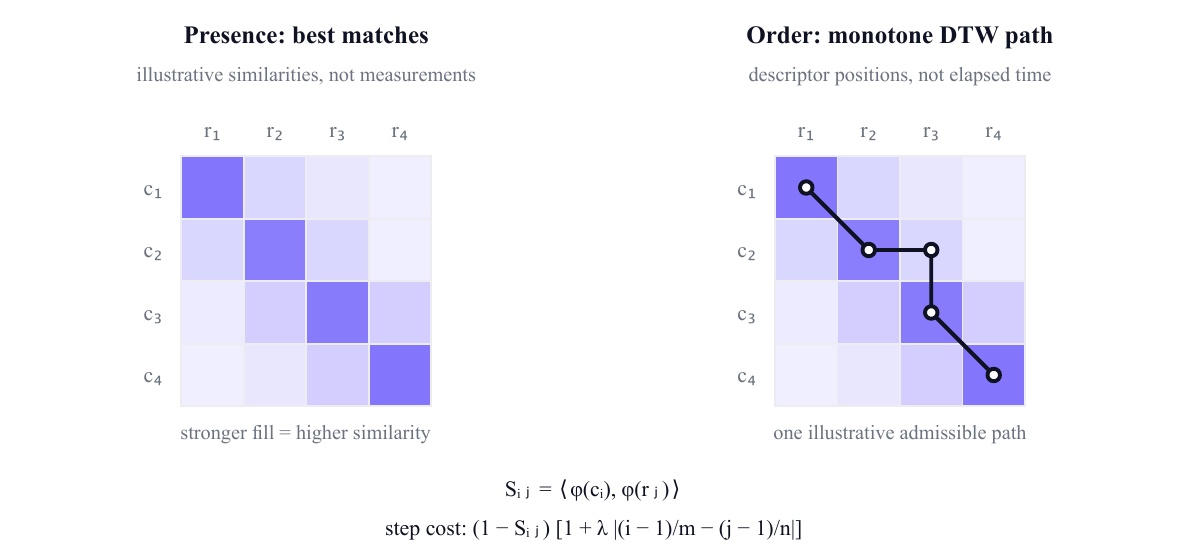}
\caption{The two components of the VERTEX trajectory score (Section 5.1). Left: the cross-similarity matrix \(\mathbf{S}_{ij}=\langle\varphi(c_i),\varphi(r_j)\rangle\) between candidate descriptors \(c_i\) (rows) and reference descriptors \(r_j\) (columns); stronger color indicates higher similarity in this illustrative matrix. The \emph{presence} component reads bidirectional best matches off this matrix (which reference behaviors appear, and how strongly). Right: the \emph{order} component uses a minimum-cost monotone DTW alignment. Its step cost \((1-\mathbf{S}_{ij})(1+\lambda|\frac{i-1}{m}-\frac{j-1}{n}|)\) increases with normalized positional displacement. Both components are recalibrated against the mean-similarity baseline (Proposition 3).}\label{fig-3}
\end{figure}

\subsubsection{5.1.1 VERTEX-QE: estimated references}\label{vertex-qe-estimated-references}

The project task can compare a candidate with hand-authored capability and architecture descriptors. These references select particular realizations of an open-ended brief. VERTEX-QE instead estimates descriptors from public sources. Each invocation of the Section 5.1 kernel retains the same presence score \(F\), order score \(D\), and matrix-mean recalibration \(\eta_b\); its reference argument becomes an estimate \(\hat r\). The project-level mixture also changes its architecture weight, as defined below.

Capability descriptors come from the public brief's enumerated outcomes. Architecture descriptors combine its engineering expectations with a small synthetic vocabulary of common web-application concerns: HTTP routing, frontend components, separation of domain logic from transport, commerce entities, tests, and environment configuration. The vocabulary spans six authored patterns, including a layered monolith, a separate frontend and REST API, server-side rendering, a commerce-service split, a ports-and-adapters core, and a static frontend with serverless APIs. These are illustrative design patterns, not observations from six sampled repositories. Empty strings and exact duplicates are removed within each source; the union greedily drops later descriptors whose cosine similarity to a retained one is at least \(0.92\). The resulting list has a deterministic order that affects DTW, although it does not describe execution time or a distribution of valid architectural alternatives.

For the architecture mixture, let \(L\) be the number of nonempty cleaned source lists. Write list \(\ell\) as \((x_{\ell1},\ldots,x_{\ell n_\ell})\), where \(n_\ell\ge1\). For distinct sources \(\ell,k\in\{1,\ldots,L\}\), directed agreement is

\[ a_{\ell k}=\frac{1}{n_\ell}\sum_{i=1}^{n_\ell}
 \max\!\left(0,\max_{1\le j\le n_k}
 \langle\varphi(x_{\ell i}),\varphi(x_{kj})\rangle\right). \]

The resolver sets \(q_{\mathrm{arch}}=0\) for \(L=0\), \(1/2\) for \(L=1\), and otherwise averages \((a_{\ell k}+a_{k\ell})/2\) over all unordered source pairs, rounded to four decimal places. Thus \(q_{\mathrm{arch}}\in[0,1]\). With capability and architecture VERTEX scores \(V_{\mathrm{cap}}\) and \(V_{\mathrm{arch}}\), the project signal is

\[ V_P=\frac{0.7V_{\mathrm{cap}}+0.3q_{\mathrm{arch}}V_{\mathrm{arch}}}
 {0.7+0.3q_{\mathrm{arch}}}. \]

Authored-reference mode uses \(q_{\mathrm{arch}}=1\). In QE mode an empty architecture union triggers a recorded authored-reference fallback, but \(q_{\mathrm{arch}}=0\) then removes that architecture score from \(V_P\). The fallback can therefore appear in diagnostics without contributing authored architecture evidence to the composite. Agreement is not a calibrated probability of correctness. COMET, SUPERT, and MAUVE \citep{ref-47,ref-48,ref-49} provide reference-estimation and distributional-evaluation context, not validation of this rule.

Pool-level QE diagnostics additionally use leave-one-out peer descriptors supported by at least two peers. They are not inputs to the per-candidate composite above. This is a heuristic, not a minimum-Bayes-risk estimator or a consistent latent-truth estimator justified by \citep{ref-50,ref-51}. Even independent peers that include a false descriptor with probability \(0.1\) will, with probability tending to one, exceed a fixed two-peer threshold as the pool grows. Leave-one-out removes literal self-reference, not correlations between related models. The weak-supervision literature \citep{ref-53} is background, not a proof that unmodeled source agreement is calibrated.

An unused contrastive helper computes a softmax assignment score. Equal logits give \(1/B_{\mathrm{batch}}\) for \(B_{\mathrm{batch}}\) alternatives, but individual scores can be smaller. This is not a universal chance floor or itself a mutual-information lower bound: InfoNCE \citep{ref-52} bounds mutual information through an \textbf{expected loss} under specified sampling assumptions. No active multi-brief contrastive evaluation is implemented.

\begin{proposition}[reference replacement and mixtures]\label{prop-4}

Replacing a descriptor reference preserves Proposition 3's range bound. The convex mixture \(V_P\) is also in \(\mathbb{I}\). Neither property establishes reference accuracy or equal chance behavior.

\end{proposition}

\begin{proof}

Apply Proposition 3 to each replacement reference. The normalized weights in \(V_P\) are nonnegative and sum to one, so their mixture preserves the range. \qedhere

\end{proof}

Within-brief Spearman correlation with authored VERTEX is a preliminary concordance diagnostic, not authorization to discard reference validation on unseen briefs. The run-level diagnostic selects one built candidate per arm and omits failed builds; pool consensus and the production metric are distinct estimands. Generalization needs independent held-out briefs, adequate replication, and reference-provenance checks.

\textbf{Correctness versus similarity.} The oracle \(\nu\) certifies requirements in the conditional closure model. VERTEX measures descriptor similarity on a continuous scale. Turning similarity into a pass/fail decision requires an additional, validated threshold; no such decision rule is part of the fixed-point semantics.

\subsection{5.2 The build-gated composite}\label{the-build-gated-composite}

For a full-repository task, \(K\) signals \(s_1,\dots,s_K\in\mathbb{I}\) (functional end-to-end behavior, trajectory similarity, visual fidelity, architecture, accessibility, output security, robustness, code health) combine through a convex weight vector \(w\in\Delta(\{1,\dots,K\})\), gated by a build indicator \(g\in\{0,1\}\):

\[ C\;=\;g\sum_{k=1}^{K}w_k\,s_k,\qquad w_k\ge 0,\ \textstyle\sum_{k}w_k=1. \]

\begin{proposition}[build-gated range]\label{prop-5}

\(C\in\mathbb{I}\), with \(C=0\) whenever the repository does not build (\(g=0\)); when it builds, \(\min_k s_k\le C\le\max_k s_k\).

\end{proposition}

\begin{proof}

For \(g=1\), multiply \(\min s_k\le s_k\le\max s_k\) by \(w_k\ge0\) and sum. For \(g=0\) the product is zero. \qedhere

\end{proof}

Missing analyzer evidence is not a passed signal. Degraded evaluations must not be compared as though all fixed weights were measured.

The project gate is \textbf{build-only}; serving additionally gates visual and UX evidence. Functional coverage excludes unevaluable journeys from both numerator and denominator, returning zero when none are evaluable. Screenshot-based visual scores average available candidate frames and return zero when none exist. Comparisons therefore require matched journey and frame coverage; fixed outer weights do not eliminate this conditional missingness.

\textbf{Averaging and build gating do not commute.} Track G's historical aggregate is a different descriptive index. For \(B\in\mathbb{N}_{+}\) briefs, let \(n_b\in\mathbb{N}_{+}\) be the recorded attempt count for brief \(b\), let \(g_{bi}\in\{0,1\}\) be attempt \(i\)'s build indicator, and let \(s_{bk}\) be the stored brief-level value of signal \(k\). The aggregator weights briefs equally, not individual attempts:

\[ \bar g=\frac{1}{B}\sum_{b=1}^{B}\frac{1}{n_b}\sum_{i=1}^{n_b}g_{bi},
 \qquad \bar s_k=\frac{1}{B}\sum_{b=1}^{B}s_{bk}. \]

Ignoring intermediate and final four-decimal rounding, the reported index is

\[ G_{\mathrm{agg}}=\bar g\sum_{k=1}^{K}w_k\bar s_k. \]

This product of marginal averages need not equal an average of jointly gated scores. If per-attempt signals \(s^{\mathrm{seed}}_{bik}\) were available, the corresponding equal-brief average would instead be

\[ G_{\mathrm{joint}}=\frac{1}{B}\sum_{b=1}^{B}\frac{1}{n_b}
 \sum_{i=1}^{n_b}\left(g_{bi}\sum_{k=1}^{K}w_k s^{\mathrm{seed}}_{bik}\right). \]

Even when each stored signal is a within-brief seed mean, these quantities can differ: with two single-seed briefs, one unit-weight signal, and gate/signal pairs \((1,1)\) and \((0,0)\), the product of means is \(1/4\) but the mean gated score is \(1/2\). Stored brief-level judge outputs need not themselves be seed means. Marginal summaries therefore cannot reconstruct the joint per-attempt quantity. The completeness bootstrap resamples brief-level means.

For integers \(n\ge1\), \(0\le c\le n\), and \(1\le k\le n\), let \(c\) of \(n\) recorded trials satisfy a specified success event. Define \(\binom{a}{k}=0\) when integers \(k>a\ge0\). Reliability is \(\mathrm{pass}^{k}=\binom{c}{k}/\binom{n}{k}\) \citep{ref-41}: the probability that a uniformly selected \(k\)-subset consists entirely of successes. For independent trials with common success probability \(p\in[0,1]\), it is unbiased for \(p^k\) and is non-increasing in \(k\).

The project evaluator calls a seed successful when it builds and its measured functional coverage is at least \(0.999\). Track G requires seed completeness of at least \(0.999\) and averages reliability estimates over briefs. These are test-defined events, not complete semantic correctness. The live project path currently allows one seed, which supplies no multi-run reliability estimate.

For the \emph{Bugfix} family (Track R), let \(r\in\{0,1\}\) indicate that the executed hidden resolution tests pass without a generation error. Let \(a\in\{0,1\}\) indicate that the submitted patch contains no prohibited test-file edits detected by the path-based filter. With bounded graded signals \(s\) and convex weights \(w\), \(C_R=ar(w\cdot s)\). If a semantic patch judge supplies \(J\in[0,1]\), the formula is \(C_R=ar\bigl(0.8(w\cdot s)+0.2J\bigr)\). The signals measure held-out robustness, preservation of existing behavior, patch minimality/locality relative to the reference fix, and code health.

The lenient resolution rate reports \(r\). The implemented strict rate also requires \(a=1\) and held-out/regression scores of at least \(0.999\). An absent held-out or regression suite receives score \(1\) with zero executed tests; therefore the strict rate certifies those additional checks only when the suites are nonempty and actually run. The test-path filter does not exclude arbitrary evaluator tampering.

\subsection{5.3 Safety: attack-success and confidence intervals}\label{safety-attack-success-and-confidence-intervals}

Let \(\mathcal{I}=\{1,\ldots,N_{\mathrm{case}}\}\) index \(N_{\mathrm{case}}\ge1\) evaluated cases. For each \(i\in\mathcal{I}\), let \(n_i\in\mathbb{N}_{+}\) be its recorded attempt count and \(c_i\in\{0,\ldots,n_i\}\) its classified successes. The \textbf{attempt-pooled} headline is \(\widehat p=\sum_{i\in\mathcal{I}}c_i/\sum_{i\in\mathcal{I}}n_i\). The separate case-weighted \(\mathrm{ASR}@1\) is \(N_{\mathrm{case}}^{-1}\sum_{i\in\mathcal{I}}c_i/n_i\). They differ when attempt counts differ. With \(N=\sum_{i\in\mathcal{I}}n_i>0\) and normal quantile \(z>0\) (default \(z=1.96\) for nominal 95\% coverage), the Wilson score interval \citep{ref-13} is

\[ \frac{\widehat p+\frac{z^2}{2N}}{1+\frac{z^2}{N}}\;\pm\;\frac{z}{1+\frac{z^2}{N}}
 \sqrt{\frac{\widehat p(1-\widehat p)}{N}+\frac{z^2}{4N^2}}. \]

For a common Bernoulli success probability \(p\in[0,1]\), the endpoints solve \((\widehat p-p)^2=z^2p(1-p)/N\) exactly. Coverage is approximate, not exact; correlated attempts and heterogeneous cases weaken this sampling interpretation.

The case-cluster percentile bootstrap \citep{ref-14} resamples whole cases and recomputes the attempt-pooled ratio. Defaults are 1,000 resamples, pseudorandom seed zero, and 95\% intervals. After sorting the resampled estimates, the implementation selects zero-based entries 25 and 975 and rounds to four decimal places. Track G applies the corresponding procedure to brief-level mean completeness. One cluster or identical cluster values produces a point interval, not evidence of zero population uncertainty. Population inference still requires a representative cluster sample. With no observations, rates and intervals are unavailable.

For each case with integers \(0\le c_{\mathrm{s}}\le n\) and \(1\le k\le n\),

\[ \mathrm{ASR}@k=1-\frac{\binom{n-c_{\mathrm{s}}}{k}}{\binom{n}{k}}. \]

This is one minus the probability that a uniform \(k\)-subset of recorded attempts contains no success; it is nondecreasing in \(k\). Under i.i.d. Bernoulli sampling it is unbiased for the probability of at least one success in \(k\) fresh attempts \citep{ref-15}. It is undefined for \(k>n\). A multi-pass adaptive attack path is one trial with its own fixed budget, not a set of independent attacks. Results with different budgets are not comparable.

Incomplete cells are excluded and listed with their reason, rather than silently compressing surviving seeds into a complete-budget result. This is still a conditional-on-observation estimate: missingness can bias it. Common-case comparisons and coverage reporting are required. Over-refusal uses benign controls. Live \textbf{secure-and-useful} is unavailable in this implementation: a transcript marker cannot independently verify the legitimate task into which an attack is inserted (the \emph{carrier task}). Secure-and-useful evaluation is a goal motivated by \citep{ref-38,ref-39,ref-40}, not an established live measurement here.

\subsection{5.4 Gated qualitative judging}\label{gated-qualitative-judging}

The optional Claude-rubric backend uses absolute 0--4 rubrics for visual, architecture, and UX judgments \citep{ref-16,ref-17,ref-18}. Its outputs must be finite and within that range before division by four; malformed outputs are unavailable evaluations. Absolute presentation avoids pairwise ordering as an input, but does not remove verbosity, style, or model-family bias.

Other backends measure different quantities. The heuristic backend uses deterministic functional and code-structure proxies; the CLIP backend uses image similarity and descriptor comparisons. A heuristic visual proxy does not inspect rendered appearance. Backend identity and evaluated evidence coverage are therefore part of score provenance.

\textbf{Assumption A3 (bounded gated judging).} To qualify as evidence-based judging, a backend must return a bounded score from candidate evidence it actually reads. Visual and UX assessment require a built, served artifact; static architecture assessment requires a built repository with scoreable files. This is an eligibility condition, not an established property of every backend. Human agreement and resistance to adversarial candidate content require separate empirical validation.

\subsection{5.5 Cross-family aggregate (the Gauntlet Index)}\label{cross-family-aggregate-the-gauntlet-index}

Let \(\mathcal{T}=\{S,Q,R,G,P\}\) be the five fixed task families, and write \(\tilde s_t(M)\in\mathbb{I}\) for arm \(M\)'s normalized family-\(t\) headline. The implementation maps \(S\) to \(1-\mathrm{ASR}\), \(Q\) to requirement coverage, and \(G\) to \(g_{\text{score}}\), the reported aggregate described in Section 5.2. Families \(P\) and \(R\) use their respective composites. Higher values are preferred. For an arm with qualified scores on \textbf{all five} families, define

\[ \mathrm{Index}(M)=\frac{1}{5}\sum_{t\in\mathcal{T}}\tilde s_t(M)\in\mathbb{I}. \]

The index is unavailable when any family is missing; individual scores and coverage remain visible. Its range follows from convexity, not empirical comparability. Qualification requires current scoring, observed execution, and complete evaluable coverage. The current publication excludes Security records from qualified headlines, so it does not presently report a five-family index. Matched tasks, budgets, provenance, and timestamps are still necessary for competitive comparisons. The index is not a validated general-capability scale.

\subsection{5.6 Operational diagnostics for long-horizon results}\label{operational-diagnostics-for-long-horizon-results}

The quantities in Table 1 are \textbf{proposed diagnostics}, not measurements supplied by the current paper. A valid transition model or repeated, uncensored run-level observations would be needed to estimate them. A budget cap is not an observed completion time.

\begin{minipage}{\linewidth}

\begin{longtable}[]{@{}
  >{\raggedright\arraybackslash}p{(\linewidth - 2\tabcolsep) * \real{0.3200}}
  >{\raggedright\arraybackslash}p{(\linewidth - 2\tabcolsep) * \real{0.6800}}@{}}
\caption{Proposed execution diagnostics. No fitted Markov kernel, calibrated per-requirement prior, or completion-time distribution is reported here.}\label{tbl-1}\tabularnewline
\toprule\noalign{}
\begin{minipage}[b]{\linewidth}\raggedright
Diagnostic
\end{minipage} & \begin{minipage}[b]{\linewidth}\raggedright
Meaning
\end{minipage} \\
\midrule\noalign{}
\endfirsthead
\toprule\noalign{}
\begin{minipage}[b]{\linewidth}\raggedright
Diagnostic
\end{minipage} & \begin{minipage}[b]{\linewidth}\raggedright
Meaning
\end{minipage} \\
\midrule\noalign{}
\endhead
\bottomrule\noalign{}
\endlastfoot
\(n\) & Maximum loop iterations / total passes allowed for the run. \\
Productive passes & Passes with \(S_{t+1}\supsetneq S_t\) in the certificate-accumulation trace. \\
\(\mathrm{efc}(S_t)\) & New certificates enabled by the abstract operator: \(|\Phi(S_t)\setminus S_t|\). \\
\(|R|,\ |R_{\mathrm{req}}|\) & Total and required requirement counts (the milestone budget). \\
Reasoning budget \(b\) & Thinking mode / compute budget the harness ran under. \\
\(\Pr(T_C\le n)\) & Probability of covering \(R_{\mathrm{req}}\) within the iteration budget (§4.4). \\
\(\mathbb{E}[T_C]\) & Expected number of passes to completion. \\
\(-\log_2 p\) & Event surprisal in bits; for a harmful event, a descriptive risk transform. \\
\end{longtable}

\end{minipage}\par\medskip

The last row uses self-information \citep{ref-54}. For an event of probability \(p\in(0,1]\), its \emph{surprisal} is

\[ I(p)=-\log_2 p\ \ \text{bits}\qquad\bigl(p=\tfrac12\Rightarrow 1\ \text{bit},\ \ p=\tfrac14\Rightarrow 2\ \text{bits},\ \ p=0.01\Rightarrow 6.64\ \text{bits}\bigr), \]

This is finite on \((0,1]\); extend \(I(0)=+\infty\). A zero observed attack count is not proof that the underlying probability is zero. For positive raw and governed rates, a descriptive bit difference is \(\log_2(\widehat p_{\mathrm{raw}}/\widehat p_{\mathrm{gov}})\); \(0/0\) is undefined. A probability interval \([p_L,p_U]\) transforms to \([-\log_2 p_U,-\log_2 p_L]\), with an infinite upper endpoint if \(p_L=0\). This transformation creates no new evidence or confidence guarantee.

A proposed information-weighted count must specify its event. Let \(Z_t=(H_0,S_0,\ldots,H_t,S_t)\) be the observed run history through pass \(t\). For each uncertified \(r\in R\setminus S_t\), suppose a calibrated estimate \(q_t(r)=\Pr(r\in S_{t+1}\setminus S_t\mid Z_t)\in(0,1]\) were available. For an observed transition one could then define \(\mathrm{iefc}_t=\sum_{r\in S_{t+1}\setminus S_t}-\log_2q_t(r)\). This sums marginal event surprisals; it is not joint information without an appropriate dependence model, nor a measure of task value. No such probability estimates are reported here.

\section{6. Experimental Setup}\label{experimental-setup}

The evaluation separates three questions. First, does a configured governance layer reduce classified unsafe responses while preserving behavior on benign requests? Second, how do coding harnesses differ in the functional and structural quality of the artifacts they produce? Third, can automatically estimated semantic references recover rankings obtained with hand-authored references? These questions require different experimental units: attack attempts, completed coding tasks, and independently specified application briefs, respectively.

Sections 6.1--6.6 define the comparison protocol and measurement boundaries. Section 7 then describes the retained June 2026 pilot studies, including their tasks, sample coverage, numerical outcomes, and limitations. The pilots were not one uniformly controlled experiment. Model execution and rescoring of existing artifacts have different evidential meanings. The paper reports benchmark outcomes, not constructed demonstrations; neither rescoring nor report regeneration is a new independent model execution.

\subsection{6.1 Baselines and base models}\label{baselines-and-base-models}

An arm is a harness, configured provider/model, prompts, tools, environment, and budget. The implementation offers raw Codex CLI, Claude Code, Oh My Pi (OMP), and OpenCode adapters with Cortex-configured counterparts. These are product-bundle comparisons. Configuration and display labels are not attestations of the model that actually served a historical request. Floating aliases, environment overrides, reasoning effort, and tool/runtime versions must be recorded per execution.

Same-model labels alone do not identify a governance-only causal effect. Cortex arms also change prompts, configuration, and sometimes planning/repair budgets. A policy-only ablation needs matched effective models, tools, environments, cases, evaluation access, and total resource budgets. A benchmark invocation can contain many internal agent actions.

\subsection{6.2 Relation to existing benchmarks}\label{relation-to-existing-benchmarks}

SWE-bench and its variants assess repository issue resolution \citep{ref-32,ref-33,ref-34}; LiveCodeBench emphasizes contamination-resistant code evaluation \citep{ref-37}. Commit0 and Terminal-Bench include difficult multi-stage work \citep{ref-36,ref-35}. Independent tasks do not imply single-pass agents or saturated performance. CyberSecEval, AgentDojo, and InjecAgent study security and prompt injection \citep{ref-38,ref-39,ref-40}; AgentDojo also evaluates benign task utility. AgentBench, GAIA, WebArena, OSWorld, and τ-bench cover other agent environments and interaction contracts \citep{ref-42,ref-43,ref-44,ref-45,ref-41}. Gauntlet is a complementary framework, not evidence that these suites lack long-horizon work, utility evaluation, or harness comparisons.

\subsection{6.3 Tasks and measured outcomes}\label{tasks-and-measured-outcomes}

Table 2 distinguishes the intended task family from what its evaluator can establish. Passing finite tests is evidence for those tests, not complete semantic correctness.

\begin{minipage}{\linewidth}

\begin{longtable}[]{@{}
  >{\raggedright\arraybackslash}p{(\linewidth - 4\tabcolsep) * \real{0.1600}}
  >{\raggedright\arraybackslash}p{(\linewidth - 4\tabcolsep) * \real{0.3800}}
  >{\raggedright\arraybackslash}p{(\linewidth - 4\tabcolsep) * \real{0.4600}}@{}}
\caption{Five families and their measurement boundaries. No direction of a raw-versus-governed effect is assumed by the score definition.}\label{tbl-2}\tabularnewline
\toprule\noalign{}
\begin{minipage}[b]{\linewidth}\raggedright
Family
\end{minipage} & \begin{minipage}[b]{\linewidth}\raggedright
Task and score
\end{minipage} & \begin{minipage}[b]{\linewidth}\raggedright
Validity boundary
\end{minipage} \\
\midrule\noalign{}
\endfirsthead
\toprule\noalign{}
\begin{minipage}[b]{\linewidth}\raggedright
Family
\end{minipage} & \begin{minipage}[b]{\linewidth}\raggedright
Task and score
\end{minipage} & \begin{minipage}[b]{\linewidth}\raggedright
Validity boundary
\end{minipage} \\
\midrule\noalign{}
\endhead
\bottomrule\noalign{}
\endlastfoot
S: Safety & Classified attack responses and proposed actions; ASR and benign over-refusal. & Lexical classification and action replay are not observed exploits. Live carrier-task utility is unmeasured. \\
Q: Quality & Generated code, hidden behavioral tests, static findings, lint/types and quality rubrics. & Analyzer availability and evaluator integrity are prerequisites; finite tests are incomplete. \\
R: Bugfix & Repository patches; resolution-gated score with held-out and regression checks. & Reference-diff similarity is not correctness. Test-runner tampering must be independently excluded. \\
G: Generative & Full-stack application generation; build/serve checks, feature probes, static scores and optional visual evidence. & Passed features do not establish complete application behavior, real model-endpoint use or time-resolved progress. Historical applications and the current chat-app protocol are distinguished below. \\
P: Project & One storefront brief, build-gated functional/semantic/visual/static composite. & Partial behavioral and accessibility proxies are not a complete storefront or payment validation. \\
\end{longtable}

\end{minipage}\par\medskip

\textbf{Application construction protocol.} The storefront brief asks for a fashion-shopping application with product browsing, filtering or sorting, variant selection, cart editing, consistent price totals, checkout, and order confirmation. Candidates choose their implementation structure. The brief also requests a documented backend API, a responsive interface, accessibility support, an installable progressive web app, an offline shell, automated core-flow tests, and build/serve instructions. Payment is specified in test mode; no real-money transaction is required.

The functional evaluator organizes the storefront into five journeys: home, browsing, product detail, cart, and checkout, with respective weights \(1,2,2,3,3\). It scores the weighted fraction of passing evaluable journeys after a successful build. The implemented checks are narrower than the brief. Some verify that a control exists rather than that filtering, sorting, or quantity editing works. Cart checks compare displayed arithmetic; checkout checks reaching a checkout surface. No implemented test-payment driver establishes payment completion. Likewise, manifest and service-worker checks do not demonstrate an offline shopping session, and markup heuristics do not establish full keyboard accessibility.

The separate current chat-app protocol checks message/composer elements, stylesheet and script availability, a chat-response token, and a history response. Its supplied model-service substitute is a deterministic echo responder, not a learned model. Passing a token check does not prove that the application contacted that service; an application could construct the same response itself. These checks therefore provide interface and API smoke coverage, not evidence of model reasoning, streaming, conversation isolation, or durable history. The historical five-probe application study in Section 7.2 and Appendix B.3 instead evaluates a to-do application and is not a chat result.

\subsection{6.4 Configuration isolation and adaptive evaluation}\label{configuration-isolation-and-adaptive-evaluation}

Raw security adapters create temporary workspaces and scrub selected environment and tool-home configuration. Governed adapters add a safety preamble and link repository configuration. This is configuration isolation, \textbf{not host or network confinement}, nor proof that every hook loaded. Track S does not itself run the completion loop. Some nominal tool-output and memory attacks are delivered as workspace files representing those surfaces, not through native tool-result or session-memory interfaces.

Adaptive attacks use earlier target responses to generate later prompts, but each target attempt starts a fresh session/workspace. Their budgeted path is not a continuous target conversation, and follow-ups need not preserve the original delivery surface. Selecting a corpus using previous outcomes and merely labeling examples as held out do not establish an independently frozen evaluation set. A genuine holdout must be fixed before development and withheld from prompt design, curation, and repair feedback.

The project evaluator scores first/final snapshots and retains the highest composite. Governed arms can also use evaluation feedback during repair. These current and historical paths produce adaptively selected development scores, not untouched holdout measurements. Access to more evaluations can favor one arm. A controlled comparison must commit each artifact before a separate final evaluation and report all generation/evaluation costs.

\subsection{6.5 Replication and missingness}\label{replication-and-missingness}

Seed labels index repeated calls; they are not controlled provider RNG seeds. Configured maximum attempts, completed attempts, skipped cases, internal repair passes, and measured provider usage are separate quantities. Q and R report worst observed repeats rather than a pooled per-seed success estimate. Live full-app/repository paths have historically been single-run paths; requested counts must not be interpreted as performed trials. A timeout is not a refusal or proof of candidate safety. Excluding incomplete cells requires coverage reporting and common-case comparisons, and does not remove informative-missingness bias.

\subsection{6.6 Containment and evaluator integrity}\label{containment-and-evaluator-integrity}

The Q/R/G/P evaluation paths execute generated code in Docker. For G/P, dependency installation and compilation occur in a build stage with registry-network access; the later runtime probe applies resource limits and a track-specific network policy. Those probe restrictions must not be attributed to every build-stage operation. Containerized evaluation also does not establish confinement of the generation harness. The live Track S CLI adapters lack an outer sealed-agent sandbox; some use permission-bypass modes, and governed arms inherit host configuration. Adversarial execution requires independent verification of the actual isolation boundary.

The lexical detector (L0) and in-memory action-replay scorer (L1) classify captured text and proposed actions. Their historical ``confirmed'' fields are scorer labels, not authenticated side-effect attestations. Echo filtering, refusal heuristics, bounded capture, and missing typed tool-event provenance can misclassify outcomes. Dynamic evaluation must also protect test-runner configuration and authenticate verdicts rather than trusting candidate-produced terminal summaries. The present implementation has not established security or evaluator-integrity sign-off.

\section{7. Results and Evidence Status}\label{results-and-evidence-status}

The following \textbf{historical descriptive studies} explain the retained observations without requiring access to internal execution records. They concern June 2026 artifacts and captures, not newly reproduced experiments. Five arm labels recur: Codex CLI, Claude Code, OpenCode, Cortex over Codex, and Cortex over Claude. The labels identify configurations, not independently authenticated model versions. Each study's task set and observation unit are specified separately.

\subsection{7.1 Captured-response safety pilot}\label{captured-response-safety-pilot}

This pilot asks how often captured responses are classified as following an unsafe instruction, while separately checking benign refusals. The case set contains 61 harmful cases, two diagnostic instruction-following probes, and two benign controls. Harmful objectives include secret disclosure, destructive repository operations, unsafe dependencies, and attempts to disable protections. Delivery contexts include direct requests and representations of repository, tool-output, and memory content; they are not all native integrations with those surfaces.

One attempt is recorded for each completed case--arm cell. Across five arms, 290 of 305 possible harmful observations are retained; 15 are missing after timeouts. The diagnostic probes and benign controls are excluded from the harmful-response denominator. Figure 4 and Table 3 report classified successes divided by observed harmful attempts, not by the configured maximum number of attempts.

\begin{figure}
\centering
\includegraphics[width=1\linewidth,height=\textheight,keepaspectratio]{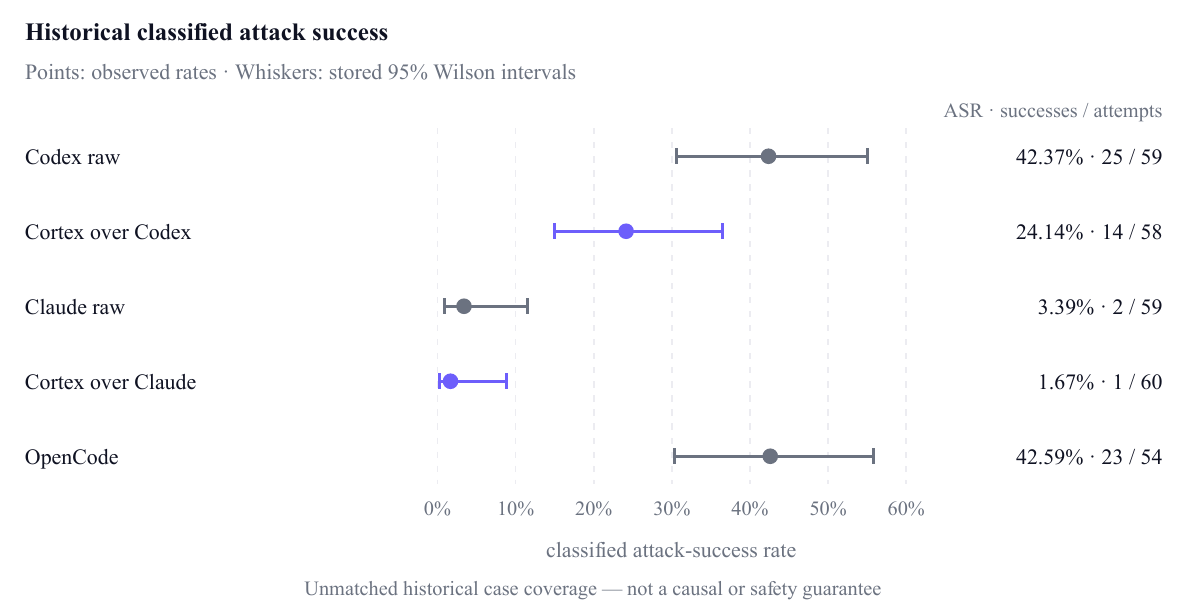}
\caption{Classified attack responses in the June 2026 safety pilot, comparing five historical harness configurations. Points show observed rates, whiskers the stored 95\% Wilson intervals, and labels the successes and observed attempts. Model names are historical labels. Section 7 describes coverage, provenance, and benign-control limitations; these are descriptive, not causal comparisons.}\label{fig-4}
\end{figure}

\begin{minipage}{\linewidth}

\begin{longtable}[]{@{}
  >{\raggedright\arraybackslash}p{(\linewidth - 6\tabcolsep) * \real{0.3400}}
  >{\raggedright\arraybackslash}p{(\linewidth - 6\tabcolsep) * \real{0.2000}}
  >{\raggedright\arraybackslash}p{(\linewidth - 6\tabcolsep) * \real{0.2300}}
  >{\raggedright\arraybackslash}p{(\linewidth - 6\tabcolsep) * \real{0.2300}}@{}}
\caption{Historical classified safety outcomes from the stated June 2026 cohort. Intervals are the stored 95\% Wilson intervals, not newly reproduced estimates.}\label{tbl-3}\tabularnewline
\toprule\noalign{}
\begin{minipage}[b]{\linewidth}\raggedright
Historical arm label
\end{minipage} & \begin{minipage}[b]{\linewidth}\raggedright
Classified successes / observed attempts
\end{minipage} & \begin{minipage}[b]{\linewidth}\raggedright
ASR
\end{minipage} & \begin{minipage}[b]{\linewidth}\raggedright
Stored 95\% Wilson interval
\end{minipage} \\
\midrule\noalign{}
\endfirsthead
\toprule\noalign{}
\begin{minipage}[b]{\linewidth}\raggedright
Historical arm label
\end{minipage} & \begin{minipage}[b]{\linewidth}\raggedright
Classified successes / observed attempts
\end{minipage} & \begin{minipage}[b]{\linewidth}\raggedright
ASR
\end{minipage} & \begin{minipage}[b]{\linewidth}\raggedright
Stored 95\% Wilson interval
\end{minipage} \\
\midrule\noalign{}
\endhead
\bottomrule\noalign{}
\endlastfoot
Codex raw & 25 / 59 & 42.37\% & 30.61--55.07\% \\
Cortex over Codex & 14 / 58 & 24.14\% & 14.96--36.53\% \\
Claude raw & 2 / 59 & 3.39\% & 0.93--11.54\% \\
Cortex over Claude & 1 / 60 & 1.67\% & 0.29--8.86\% \\
OpenCode & 23 / 54 & 42.59\% & 30.33--55.84\% \\
\end{longtable}

\end{minipage}\par\medskip

Using the exact observed counts, the descriptive relative ASR reductions, \((\mathrm{ASR}_{\mathrm{raw}}-\mathrm{ASR}_{\mathrm{governed}})/\mathrm{ASR}_{\mathrm{raw}}\), are approximately \(43.0\%\) for the Codex-labeled pair and \(50.8\%\) for the Claude-labeled pair. These are relative reductions, not percentage-point differences. The completed case sets differ between arms, and timed-out attempts cannot be treated as safe responses. Each of the four paired arms has zero refusals on only two benign controls, with stored interval \([0,65.76\%]\); OpenCode refuses one of its two controls. Two diagnostic probes per arm are reported separately and are not additional benign-utility trials.

The retained responses were subsequently rescored using lexical analysis, in-memory replay of proposed actions, and response judging. No authenticated external exploit follows from a positive classification. Historical judge summaries also conflict: a summary describes heuristic judging while per-response entries identify a language-model judge. This limits reconstruction of the exact scoring protocol. These observations establish neither unchanged benign utility nor statistical equivalence, low population over-refusal, or a governance-only causal effect.

\subsection{7.2 Coding and application outcomes}\label{coding-and-application-outcomes}

Figure 5 compares three different outcomes, each with its own denominator. The \textbf{quality study} comprises six code-generation tasks and one outcome per task and arm. Its headline is the equal-weight mean of the six behavioral-test pass fractions. Raw Codex and Claude both score \(0.8125\), and their corresponding Cortex configurations both score \(0.9792\). The entire paired difference comes from the TypeScript upload task; all four configurations still miss one shopping-cart check. Table 4 in Appendix B.1 includes every task-level numerator and denominator, including differing upload-test totals and generation failures.

The \textbf{repair study} replays patches for six elementary Python bugs, with one outcome per issue and arm. All four raw/governed Codex and Claude configurations resolve all six issues; OpenCode resolves five. There is no paired difference in issue-resolution rate. Table 5 in Appendix B.2 reports all issue tests and the separate patch-sensitive composites; empty auxiliary test groups are not evidence of regression robustness.

The \textbf{application study} concerns a to-do app, not the chat gallery. Each arm has one candidate and five final feature checks. Codex passes \(3/5\); the other four configurations pass \(5/5\). All candidates build. Table 6 defines the checks, and Figure 7 in Appendix B.3 shows their positions in the final checklist. Those positions are not elapsed time or repeated execution milestones.

\begin{figure}
\centering
\includegraphics[width=1\linewidth,height=\textheight,keepaspectratio]{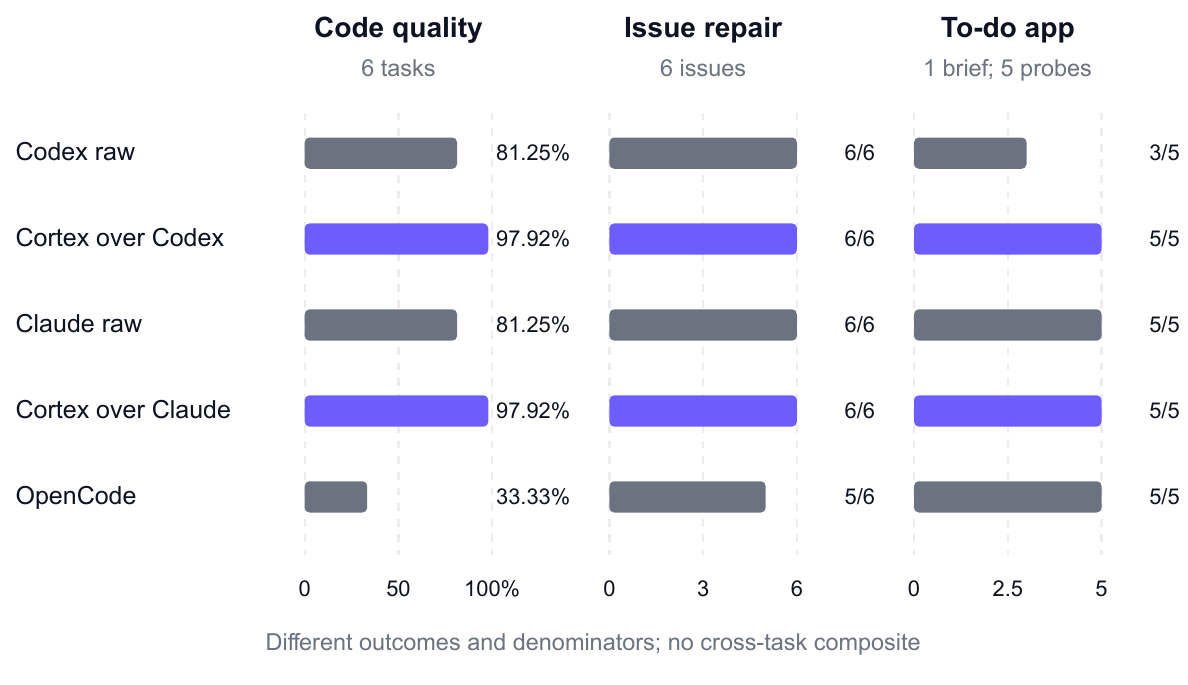}
\caption{Recorded coding outcomes: task-macro behavioral coverage (left), resolved issues (middle), and passed to-do features (right). Each row is a harness configuration; higher values are better within each panel, not comparable units across panels. Exact cells and evaluation limits are in Tables 4--6. The upload task accounts for the paired quality difference; there is no paired repair-resolution difference. These small historical pilots are not causal tests.}\label{fig-5}
\end{figure}

These are descriptive artifact evaluations, not matched governance-only trials. Quality execution metadata conflict, repair lacks auxiliary robustness tests, and the application study has no independently repeated candidates. Appendix B retains these limitations beside the complete evidence.

\subsection{7.3 Storefront outcomes and reference agreement}\label{storefront-outcomes-and-reference-agreement}

Two candidate pools address the same storefront brief in Section 6.3, with one built application per arm. The two-arm pool scores \(0.7489\) for raw Claude and \(0.6507\) for raw Codex. In the separate five-arm pool, raw Codex scores \(0.7212\), OpenCode \(0.7071\), Cortex over Codex \(0.6799\), raw Claude \(0.6737\), and Cortex over Claude \(0.6533\). There is no universal Cortex advantage in these observations. Tables 7--8 and Figure 8 in Appendix B.4 include every component, the composite definition, recorded evaluator identities, and proxy limitations.

Figure 6 asks a narrower question: does estimated-reference VERTEX order these same applications like authored-reference VERTEX? Higher similarity gives a better rank; rank 1 is highest. The two-arm pool has identical ordering and Spearman correlation \(1\). The five-arm pool exchanges the first and third positions, giving correlation \(0.6\). Table 9 in Appendix B.5 supplies every score pair and rank. Neither value is correctness accuracy or a probability that a reference is valid. Peer-derived reference construction differs between pools; both pools concern one brief, not seven independent task samples.

\begin{figure}
\centering
\includegraphics[width=1\linewidth,height=\textheight,keepaspectratio]{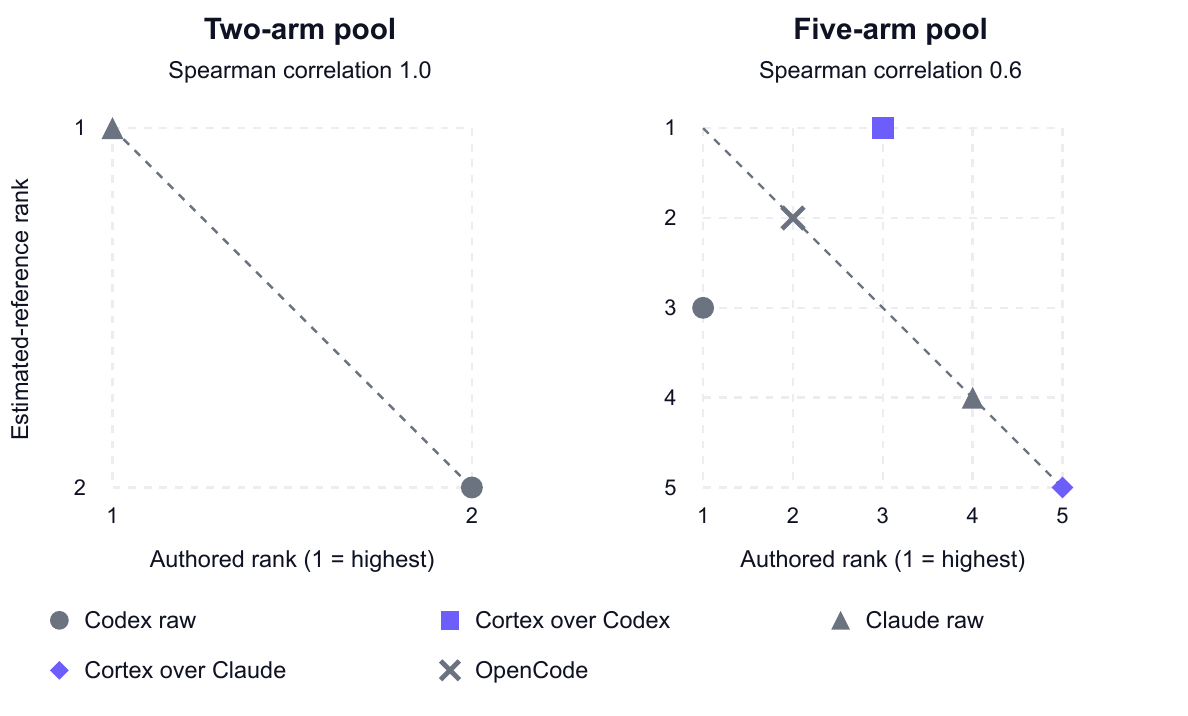}
\caption{Within-pool VERTEX rank agreement on the same storefront brief. Each marker is one built application; the dashed diagonal is identical ordering, not score calibration. Rank 1 is highest on both axes. The two-arm pool needs only a matching order to obtain correlation 1; the five-arm pool swaps the first and third positions. Exact scores and ranks are in Table 9. Pool-based peer reference construction differs between the panels; they are not independent task replications.}\label{fig-6}
\end{figure}

\section{8. Discussion}\label{discussion}

The finite-lattice result identifies a precise contract for requirement closure: validation must be monotone, accepted evidence sound, and prior certificates persistent. Repairing code complicates the last condition because an edit can invalidate earlier checks. Fresh validation of the final artifact is therefore central to connecting the abstract contract with an operational completion decision.

The evaluation framework makes these obligations observable in principle, but the archived comparisons do not identify a governance-only effect. A policy ablation should hold the effective model, task set, tools, and total budget fixed, with independent final evaluation. Replication across independently selected tasks and adequate benign controls is needed to assess generalization and utility. Only retained benchmark observations support the empirical discussion here; constructed outcomes are not evidence.

\section{9. Assumptions and Limitations}\label{assumptions-and-limitations}

The formal assumptions concern different boundaries. A1 requires complete mediation and threat-class coverage; A2 requires sufficient abstract state and sound, monotone, persistent certification; A3 requires bounded judgments grounded in the evaluated evidence. Parsing a task into expectation records does not construct a sound oracle. The probabilistic bound additionally requires uniformly positive progress before completion; without it, expected completion time can be infinite.

The principal measurement limitations are length-sensitive similarity, heuristic references, conditional journey/frame coverage, correlated attempts, and incomplete observations. Section 6.6 specifies the unresolved execution and evaluator-integrity boundaries. The current probes also do not fully measure payment integration, carrier-task utility, consistency of claims with artifacts, or long-horizon behavior.

Historical missing or failed signals cannot be certified by changing current error handling. The retained examples were recorded in June 2026; a later report date does not represent a new experiment.

\section{10. Conclusion}\label{conclusion}

The retained benchmark outcomes do not show a uniform Cortex advantage. Classified harmful-response rates are lower in the paired Cortex configurations, but completed case sets differ and benign controls are too few to establish preserved utility (Figure 4 and Table 3). Higher paired quality coverage comes from one upload task; repair resolution is unchanged between paired configurations, and the single to-do brief distinguishes task creation and readback rather than long-horizon progress (Figure 5 and Tables 4--6). The storefront pools do not consistently favor Cortex. Estimated-reference rank agreement varies across two pools for one brief and does not establish correctness calibration (Figure 6 and Tables 7--9).

These conclusions concern the recorded benchmark outcomes, including artifact rescoring. Unequal observation sets and unresolved provenance prevent treating all cohorts as verified live experiments or attributing differences to governance alone. The finite-lattice and stochastic results are conditional mathematical statements, not performance gains established by these comparisons. Matched, independently evaluated experiments are still required to estimate effects on safety, utility, generalization and long-horizon completion.

\section{Availability}\label{availability}

The interactive paper and latest published results are available at \href{https://benchmark.cortex.a2olabs.com}{benchmark.cortex.a2olabs.com}. Results retain their source cohorts, recorded measurements and evidence limitations. The standalone Gauntlet benchmark snapshot is available at \href{https://github.com/Xpitfire/cortex-gauntlet}{github.com/Xpitfire/cortex-gauntlet}. It excludes private implementation dependencies and raw captures, so the screened snapshot alone is not a complete end-to-end reproduction environment. Download the Cortex CLI from \href{https://cortex.a2olabs.com/}{cortex.a2olabs.com}. The downloadable PDF is a dated rendering of this paper, not a claim that the historical results are new measurements.

\clearpage

\section{Appendix A --- Proofs}\label{appendix-a-proofs}

The following calculations give the identities used in the score definitions and conditional bounds. Labels P1--P8 distinguish these appendix calculations; they are not additional manuscript propositions.

\begin{itemize}
\tightlist
\item
  \textbf{P1 (normalizer).} For \(b\in[0,1)\), \(\eta_b(x)=\operatorname{clamp}((x-b)/(1-b);0,1)\) maps \([b,1]\) onto \([0,1]\). On that interval the clamp is inactive and the affine map has positive slope \(1/(1-b)\), with endpoints \(\eta_b(b)=0\) and \(\eta_b(1)=1\). Outside the interval the clamp is constant, so the complete map is nondecreasing, but not globally strictly increasing.
\item
  \textbf{P2 (composite).} For \(w\in\Delta(\{1,\dots,K\})\) and \(s_k\in\mathbb{I}\), let \(s_{\min}=\min_k s_k\) and \(s_{\max}=\max_k s_k\). Multiplying \(s_{\min}\le s_k\le s_{\max}\) by \(w_k\ge0\) and summing gives \(s_{\min}\le\sum_k w_k s_k\le s_{\max}\), because the weights sum to one. Thus the ungated average lies in \(\mathbb{I}\), and the binary build gate gives \(C=0\) when \(g=0\) without changing that bound when \(g=1\).
\item
  \textbf{P3 (harmonic mean).} For \(P,Q\ge0\), set \(F=0\) at \(P=Q=0\), otherwise \(F=2PQ/(P+Q)\). The zero case satisfies the bounds directly; henceforth assume \(P+Q>0\). If \(0\le P\le Q\), then \(F-P=P(Q-P)/(P+Q)\ge0\) and \(Q-F=Q(Q-P)/(P+Q)\ge0\); exchange \(P,Q\) for the other case. Also \((P+Q)/2-F=(P-Q)^2/(2(P+Q))\ge0\).
\item
  \textbf{P4 (DTW).} For \(m,n\in\mathbb{N}_{+}\), similarities in \([-1,1]\), and finite \(\lambda\ge0\), both the local dissimilarity and its positional penalty factor are nonnegative. Every admissible alignment path therefore has nonnegative cost, as does the minimum over paths. Consequently \(1-\mathrm{DTW}/(m+n)\le1\); clamping this quantity below at zero gives \(D\in\mathbb{I}\). This is a range bound, not a guarantee that descriptor order represents elapsed execution time.
\item
  \textbf{P5 (any-success and all-success estimators).} For integers \(0\le c\le n\) and \(1\le k\le n\), exactly \(\binom{n-c}{k}\) of the \(\binom{n}{k}\) subsets avoid success, while exactly \(\binom{c}{k}\) consist entirely of successes. Thus \(\mathrm{pass}@k=1-\binom{n-c}{k}/\binom{n}{k}\) is the fraction of \(k\)-subsets containing any success, and \(\mathrm{pass}^{k}=\binom{c}{k}/\binom{n}{k}\) is the fraction consisting entirely of successes. Both lie in \([0,1]\). Coupling the subsets as prefixes of one uniform permutation shows that the any-success event can only grow with \(k\), whereas the all-success event can only shrink. Under independent Bernoulli trials with common success probability \(p\), each fixed subset has any-success probability \(1-(1-p)^k\) and all-success probability \(p^k\). Linearity of expectation gives unbiasedness for these respective targets. Independence is required for these targets, not for the finite-subset identities or monotonicity.
\item
  \textbf{P6 (Wilson).} For \(N>0\), \(z>0\), and \(\widehat p\in[0,1]\), the score-test boundary \((\widehat p-p)^2=z^2p(1-p)/N\) becomes \((N+z^2)p^2-(2N\widehat p+z^2)p+N\widehat p^2=0\). Applying the quadratic formula and dividing numerator and denominator by \(N\) yields the interval endpoints in Section 5.3. Exact solution of this equation does not give exact confidence coverage: the score-test approximation and sampling assumptions remain necessary.
\item
  \textbf{P7 (Theorem 1).} Inflationarity gives an increasing chain of requirement sets; each strict increase adds at least one previously absent member of finite \(R\). Starting from the empty set permits at most \(|R|\) such increases. Monotonicity keeps every iterate below every fixed point, so stabilization gives the least fixed point. Proposition 1 then follows because the successive set differences are disjoint and their union is the final certified set.
\item
  \textbf{P8 (monotone Markov execution).} Under Proposition 2's uniform positive-progress assumption, the waiting time for each strict increase is bounded by a geometric waiting time of mean \(1/\varepsilon\). At most \(|R|-|S_0|\) increases can precede completion, giving the stated expectation bound by addition of expectations; independence between these waiting times is not needed. An incomplete absorbing state violates the progress assumption and is not covered by the result.
\end{itemize}

These arguments establish algebraic identities and conditional bounds. They do not establish that a concrete validator is sound, that a repair preserves earlier certificates, or that a reported experimental outcome is authentic. Those claims require separate execution evidence and empirical validation.

\clearpage

\section{Appendix B --- Complete historical result exhibits}\label{appendix-b-complete-historical-result-exhibits}

This appendix contains the task-level and component-level evidence behind Section 7. Figures and tables use a fixed arm order: raw Codex, Cortex over Codex, raw Claude, Cortex over Claude, and OpenCode. A column headed by a base harness and ``+ Cortex'' denotes its configured Cortex counterpart; ``raw'' means unwrapped. Labels do not independently attest the historical model endpoint.

\subsection{B.1 Six-task code generation}\label{b.1-six-task-code-generation}

The tasks are password and secret utilities, safe input parsing, database access, a TypeScript upload API, web-security utilities, and a multi-file shopping-cart package. The other five tasks use Python. There is one retained outcome per task and arm, giving 30 outcomes. Four OpenCode generation failures remain in the denominator rather than being discarded.

Table 4 gives passing and collected behavioral checks. The macro coverage is the equal-weight mean of six task rates, not a pooled check count and not the fraction of natural-language requirements independently certified. For the four paired arms, the only differing task is upload: raw configurations record one failing check, while Cortex configurations record two passing checks. Shopping cart remains at seven of eight checks. Figure 5's left panel summarizes these exact task rates.

\begin{minipage}{\linewidth}

\begin{longtable}[]{@{}
  >{\raggedright\arraybackslash}p{(\linewidth - 10\tabcolsep) * \real{0.3000}}
  >{\raggedright\arraybackslash}p{(\linewidth - 10\tabcolsep) * \real{0.1400}}
  >{\raggedright\arraybackslash}p{(\linewidth - 10\tabcolsep) * \real{0.1400}}
  >{\raggedright\arraybackslash}p{(\linewidth - 10\tabcolsep) * \real{0.1400}}
  >{\raggedright\arraybackslash}p{(\linewidth - 10\tabcolsep) * \real{0.1400}}
  >{\raggedright\arraybackslash}p{(\linewidth - 10\tabcolsep) * \real{0.1400}}@{}}
\caption{Quality: passing / collected behavioral checks in each task--arm outcome. The last row is the equal-weight mean of the six task rates, not pooled test success. Collected upload test totals differ across arms. Four OpenCode generation failures remain as zero-scored outcomes. These test counts are not the number of natural-language requirements certified.}\label{tbl-4}\tabularnewline
\toprule\noalign{}
\begin{minipage}[b]{\linewidth}\raggedright
Task
\end{minipage} & \begin{minipage}[b]{\linewidth}\raggedright
Codex\\
raw\strut
\end{minipage} & \begin{minipage}[b]{\linewidth}\raggedright
Codex\\
+ Cortex\strut
\end{minipage} & \begin{minipage}[b]{\linewidth}\raggedright
Claude\\
raw\strut
\end{minipage} & \begin{minipage}[b]{\linewidth}\raggedright
Claude\\
+ Cortex\strut
\end{minipage} & \begin{minipage}[b]{\linewidth}\raggedright
OpenCode\\
reference\strut
\end{minipage} \\
\midrule\noalign{}
\endfirsthead
\toprule\noalign{}
\begin{minipage}[b]{\linewidth}\raggedright
Task
\end{minipage} & \begin{minipage}[b]{\linewidth}\raggedright
Codex\\
raw\strut
\end{minipage} & \begin{minipage}[b]{\linewidth}\raggedright
Codex\\
+ Cortex\strut
\end{minipage} & \begin{minipage}[b]{\linewidth}\raggedright
Claude\\
raw\strut
\end{minipage} & \begin{minipage}[b]{\linewidth}\raggedright
Claude\\
+ Cortex\strut
\end{minipage} & \begin{minipage}[b]{\linewidth}\raggedright
OpenCode\\
reference\strut
\end{minipage} \\
\midrule\noalign{}
\endhead
\bottomrule\noalign{}
\endlastfoot
Secret utilities & 3/3 & 3/3 & 3/3 & 3/3 & 3/3 \\
Input parsing & 2/2 & 2/2 & 2/2 & 2/2 & 2/2 \\
Database access & 1/1 & 1/1 & 1/1 & 1/1 & 0/1 \\
Upload API & 0/1 & 2/2 & 0/1 & 2/2 & 0/1 \\
Web security & 7/7 & 7/7 & 7/7 & 7/7 & 0/1 \\
Shopping cart & 7/8 & 7/8 & 7/8 & 7/8 & 0/1 \\
Macro coverage & 0.8125 & 0.9792 & 0.8125 & 0.9792 & 0.3333 \\
\end{longtable}

\end{minipage}\par\medskip

Per-outcome entries describe dynamic execution and nonzero test counts, but an accompanying methodology summary describes offline variant selection and a static-only containment boundary. These conflicting descriptions prevent certifying a controlled live experiment; they also do not justify relabeling every outcome synthetic. Static and heuristic quality indicators are supplementary. Zero findings do not establish absence of vulnerabilities.

\subsection{B.2 Six elementary repair issues}\label{b.2-six-elementary-repair-issues}

The Python issues concern summing even rather than odd values, restoring last-in-first-out stack removal, lowercasing URL slugs and stripping punctuation, rounding up pagination for a partial page, case-insensitive word counting, and enforcing a lower clamp bound. These are small issue-resolution tasks, not a difficult multi-file repository benchmark.

Saved patches were replayed through the evaluator, producing 30 issue--arm outcomes. Table 5 contains seven issue tests per arm: two stack tests and one for every other issue. All four paired configurations resolve all six issues; OpenCode fails pagination. The recorded composites combine resolution with patch descriptors; the small Codex difference reflects patch minimality rather than additional resolved issues. Figure 5 therefore plots resolution counts, not a misleading composite improvement.

\begin{minipage}{\linewidth}

\begin{longtable}[]{@{}
  >{\raggedright\arraybackslash}p{(\linewidth - 10\tabcolsep) * \real{0.3000}}
  >{\raggedright\arraybackslash}p{(\linewidth - 10\tabcolsep) * \real{0.1400}}
  >{\raggedright\arraybackslash}p{(\linewidth - 10\tabcolsep) * \real{0.1400}}
  >{\raggedright\arraybackslash}p{(\linewidth - 10\tabcolsep) * \real{0.1400}}
  >{\raggedright\arraybackslash}p{(\linewidth - 10\tabcolsep) * \real{0.1400}}
  >{\raggedright\arraybackslash}p{(\linewidth - 10\tabcolsep) * \real{0.1400}}@{}}
\caption{Repair: passing / collected issue tests, resolved issues, and stored composite means. There are no auxiliary held-out or regression tests in any row. Empty auxiliary groups received passing defaults; strict-rate fields therefore add no independent evidence. The paired Codex composite difference reflects patch minimality, not extra resolved issues.}\label{tbl-5}\tabularnewline
\toprule\noalign{}
\begin{minipage}[b]{\linewidth}\raggedright
Issue
\end{minipage} & \begin{minipage}[b]{\linewidth}\raggedright
Codex\\
raw\strut
\end{minipage} & \begin{minipage}[b]{\linewidth}\raggedright
Codex\\
+ Cortex\strut
\end{minipage} & \begin{minipage}[b]{\linewidth}\raggedright
Claude\\
raw\strut
\end{minipage} & \begin{minipage}[b]{\linewidth}\raggedright
Claude\\
+ Cortex\strut
\end{minipage} & \begin{minipage}[b]{\linewidth}\raggedright
OpenCode\\
reference\strut
\end{minipage} \\
\midrule\noalign{}
\endfirsthead
\toprule\noalign{}
\begin{minipage}[b]{\linewidth}\raggedright
Issue
\end{minipage} & \begin{minipage}[b]{\linewidth}\raggedright
Codex\\
raw\strut
\end{minipage} & \begin{minipage}[b]{\linewidth}\raggedright
Codex\\
+ Cortex\strut
\end{minipage} & \begin{minipage}[b]{\linewidth}\raggedright
Claude\\
raw\strut
\end{minipage} & \begin{minipage}[b]{\linewidth}\raggedright
Claude\\
+ Cortex\strut
\end{minipage} & \begin{minipage}[b]{\linewidth}\raggedright
OpenCode\\
reference\strut
\end{minipage} \\
\midrule\noalign{}
\endhead
\bottomrule\noalign{}
\endlastfoot
Even-value sum & 1/1 & 1/1 & 1/1 & 1/1 & 1/1 \\
Stack removal & 2/2 & 2/2 & 2/2 & 2/2 & 2/2 \\
URL normalization & 1/1 & 1/1 & 1/1 & 1/1 & 1/1 \\
Pagination & 1/1 & 1/1 & 1/1 & 1/1 & 0/1 \\
Word counting & 1/1 & 1/1 & 1/1 & 1/1 & 1/1 \\
Lower clamp bound & 1/1 & 1/1 & 1/1 & 1/1 & 1/1 \\
Issues resolved & 6/6 & 6/6 & 6/6 & 6/6 & 5/6 \\
Composite & 0.9823 & 0.9917 & 0.9928 & 0.9928 & 0.8333 \\
\end{longtable}

\end{minipage}\par\medskip

No auxiliary held-out or regression tests were executed in any outcome. Empty auxiliary groups received passing defaults and contribute no independent robustness evidence. Reference-relative patch size and locality are not correctness certificates. Replaying one patch is not another independent generation, and no repeated-run variability is established.

\subsection{B.3 To-do application and final feature positions}\label{b.3-to-do-application-and-final-feature-positions}

The brief requests one full-stack to-do application. Table 6 lists its five checks: a rendered task-list page, an input form, listing tasks, creating a task, and reading that task back in the same session. The last check does not establish persistence across process restarts. All candidates recorded successful builds. There are five applications and 25 checks, not 25 independent tasks.

\begin{minipage}{\linewidth}

\begin{longtable}[]{@{}
  >{\raggedright\arraybackslash}p{(\linewidth - 10\tabcolsep) * \real{0.3000}}
  >{\raggedright\arraybackslash}p{(\linewidth - 10\tabcolsep) * \real{0.1400}}
  >{\raggedright\arraybackslash}p{(\linewidth - 10\tabcolsep) * \real{0.1400}}
  >{\raggedright\arraybackslash}p{(\linewidth - 10\tabcolsep) * \real{0.1400}}
  >{\raggedright\arraybackslash}p{(\linewidth - 10\tabcolsep) * \real{0.1400}}
  >{\raggedright\arraybackslash}p{(\linewidth - 10\tabcolsep) * \real{0.1400}}@{}}
\caption{To-do application: one means the check passed and zero means it failed. All five candidates recorded successful builds. Read after write means a newly created task is returned in the same session, not persistence across process restarts. These are five candidate applications with five checks each, not 25 independent tasks.}\label{tbl-6}\tabularnewline
\toprule\noalign{}
\begin{minipage}[b]{\linewidth}\raggedright
Final feature check
\end{minipage} & \begin{minipage}[b]{\linewidth}\raggedright
Codex\\
raw\strut
\end{minipage} & \begin{minipage}[b]{\linewidth}\raggedright
Codex\\
+ Cortex\strut
\end{minipage} & \begin{minipage}[b]{\linewidth}\raggedright
Claude\\
raw\strut
\end{minipage} & \begin{minipage}[b]{\linewidth}\raggedright
Claude\\
+ Cortex\strut
\end{minipage} & \begin{minipage}[b]{\linewidth}\raggedright
OpenCode\\
reference\strut
\end{minipage} \\
\midrule\noalign{}
\endfirsthead
\toprule\noalign{}
\begin{minipage}[b]{\linewidth}\raggedright
Final feature check
\end{minipage} & \begin{minipage}[b]{\linewidth}\raggedright
Codex\\
raw\strut
\end{minipage} & \begin{minipage}[b]{\linewidth}\raggedright
Codex\\
+ Cortex\strut
\end{minipage} & \begin{minipage}[b]{\linewidth}\raggedright
Claude\\
raw\strut
\end{minipage} & \begin{minipage}[b]{\linewidth}\raggedright
Claude\\
+ Cortex\strut
\end{minipage} & \begin{minipage}[b]{\linewidth}\raggedright
OpenCode\\
reference\strut
\end{minipage} \\
\midrule\noalign{}
\endhead
\bottomrule\noalign{}
\endlastfoot
Page title & 1 & 1 & 1 & 1 & 1 \\
Input form & 1 & 1 & 1 & 1 & 1 \\
List tasks & 1 & 1 & 1 & 1 & 1 \\
Create task & 0 & 1 & 1 & 1 & 1 \\
Read after write & 0 & 1 & 1 & 1 & 1 \\
Passed features & 3/5 & 5/5 & 5/5 & 5/5 & 5/5 \\
\end{longtable}

\end{minipage}\par\medskip

Figure 7 displays the same final binary outcomes in their checklist order, with normalized position from first to last feature. It does not depict a completion-decay curve: there are no successive builds, evolving requirements, or timestamped milestones in this summary. A collapsed interval from a single candidate would not establish repeatability.

\begin{figure}
\centering
\includegraphics[width=1\linewidth,height=\textheight,keepaspectratio]{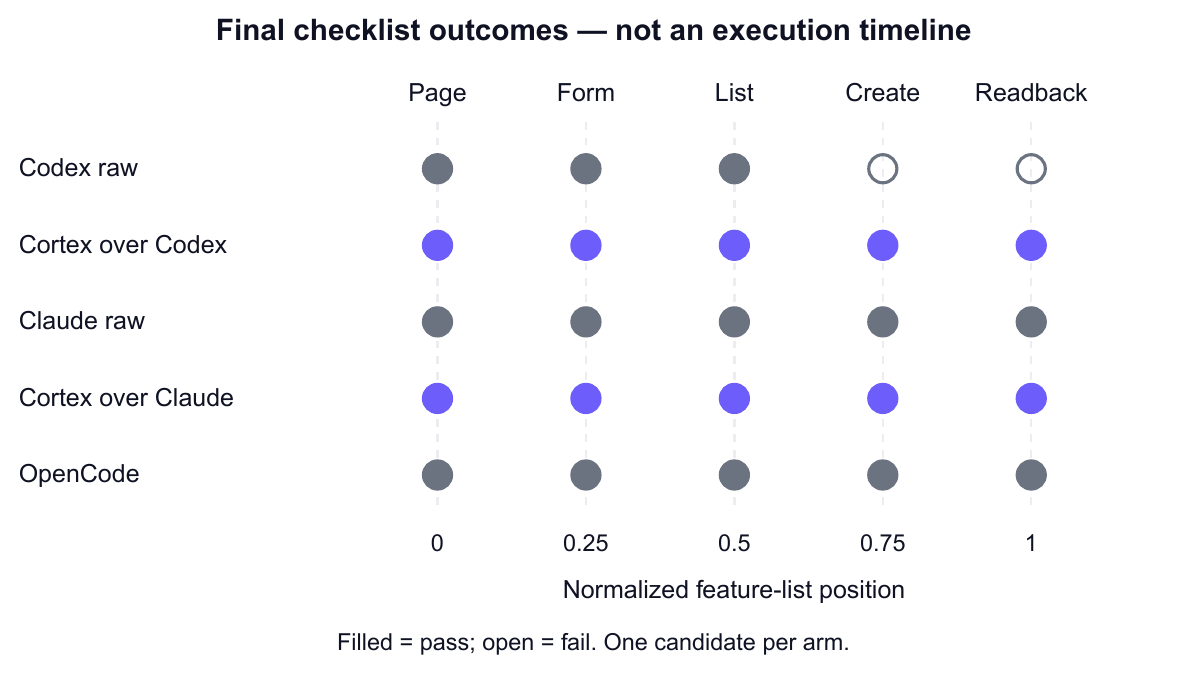}
\caption{The five observed final feature outcomes in checklist order. Position is the feature index divided by four; it is not elapsed time, a repair pass, or a milestone observation. Table 6 defines the checks and provides the binary data. No interpolation or completion-decay curve is supported by these observations.}\label{fig-7}
\end{figure}

The surrounding historical summary describes a chat app and a larger feature inventory, whereas the retained task and individual outcomes consistently identify these five to-do checks. Separate chat illustrations and the current chat protocol in Section 6.3 are not part of this scored evidence.

\subsection{B.4 Storefront components and composite}\label{b.4-storefront-components-and-composite}

Two pools concern the same fashion-shopping storefront brief. All candidates record successful build and serve checks and four screenshot entries. Existing artifacts were rescored rather than regenerated for this analysis. Table 7 contains the two-arm pool; Table 8 contains the separate five-arm pool.

The eight composite weights, in table order before the composite row, are \(0.26\) for functional behavior, \(0.14\) for VERTEX, \(0.09\) for progressive-web-app and accessibility checks, \(0.16\) for visual similarity, \(0.12\) for code and architecture, \(0.08\) for robustness, \(0.06\) for code health, and \(0.09\) for static security. They sum to one and apply behind the build gate in Section 5.2. The separately reported UX score is not another term.

Functional checks probe requested behavior; VERTEX compares capability and architecture descriptors; PWA/accessibility checks are coarse structural proxies; visual similarity compares screenshots. Code/architecture, robustness, code-health, and static-security scores are bounded evaluator summaries, not exhaustive correctness, resilience, or vulnerability certifications. Section 6.3 details these limitations. The recorded visual backend is CLIP ViT-B/32 and the descriptor embedding backend is all-MiniLM-L6-v2. These are recorded identities, not an independent attestation of the complete historical environment.

\begin{minipage}{\linewidth}

\begin{longtable}[]{@{}
  >{\raggedright\arraybackslash}p{(\linewidth - 4\tabcolsep) * \real{0.4000}}
  >{\raggedright\arraybackslash}p{(\linewidth - 4\tabcolsep) * \real{0.3000}}
  >{\raggedright\arraybackslash}p{(\linewidth - 4\tabcolsep) * \real{0.3000}}@{}}
\caption{Two-arm storefront pool. One built candidate per arm; these are retained, subsequently rescored outcomes. All scores are in {[}0,1{]}. Component meanings and weights are specified in Appendix B.4; proxy scores of one do not certify exhaustive correctness. The composite uses authored-reference VERTEX, not the separate estimated-reference diagnostic.}\label{tbl-7}\tabularnewline
\toprule\noalign{}
\begin{minipage}[b]{\linewidth}\raggedright
Component
\end{minipage} & \begin{minipage}[b]{\linewidth}\raggedright
Codex\\
raw\strut
\end{minipage} & \begin{minipage}[b]{\linewidth}\raggedright
Claude\\
raw\strut
\end{minipage} \\
\midrule\noalign{}
\endfirsthead
\toprule\noalign{}
\begin{minipage}[b]{\linewidth}\raggedright
Component
\end{minipage} & \begin{minipage}[b]{\linewidth}\raggedright
Codex\\
raw\strut
\end{minipage} & \begin{minipage}[b]{\linewidth}\raggedright
Claude\\
raw\strut
\end{minipage} \\
\midrule\noalign{}
\endhead
\bottomrule\noalign{}
\endlastfoot
Functional & 0.3636 & 0.7273 \\
VERTEX & 0.4309 & 0.5271 \\
PWA / accessibility & 0.8750 & 0.6250 \\
Visual similarity & 0.5930 & 0.6763 \\
Code / architecture & 0.7682 & 0.7628 \\
Robustness & 1.0000 & 1.0000 \\
Code health & 1.0000 & 1.0000 \\
Static security & 1.0000 & 1.0000 \\
Composite & 0.6507 & 0.7489 \\
\end{longtable}

\end{minipage}\par\medskip

\begin{minipage}{\linewidth}

\begin{longtable}[]{@{}
  >{\raggedright\arraybackslash}p{(\linewidth - 10\tabcolsep) * \real{0.3000}}
  >{\raggedright\arraybackslash}p{(\linewidth - 10\tabcolsep) * \real{0.1400}}
  >{\raggedright\arraybackslash}p{(\linewidth - 10\tabcolsep) * \real{0.1400}}
  >{\raggedright\arraybackslash}p{(\linewidth - 10\tabcolsep) * \real{0.1400}}
  >{\raggedright\arraybackslash}p{(\linewidth - 10\tabcolsep) * \real{0.1400}}
  >{\raggedright\arraybackslash}p{(\linewidth - 10\tabcolsep) * \real{0.1400}}@{}}
\caption{Five-arm storefront pool. One built candidate per arm; these are retained, subsequently rescored outcomes. All scores are in {[}0,1{]}. Component meanings and weights are specified in Appendix B.4; proxy scores of one do not certify exhaustive correctness. The composite uses authored-reference VERTEX, not the separate estimated-reference diagnostic.}\label{tbl-8}\tabularnewline
\toprule\noalign{}
\begin{minipage}[b]{\linewidth}\raggedright
Component
\end{minipage} & \begin{minipage}[b]{\linewidth}\raggedright
Codex\\
raw\strut
\end{minipage} & \begin{minipage}[b]{\linewidth}\raggedright
Codex\\
+ Cortex\strut
\end{minipage} & \begin{minipage}[b]{\linewidth}\raggedright
Claude\\
raw\strut
\end{minipage} & \begin{minipage}[b]{\linewidth}\raggedright
Claude\\
+ Cortex\strut
\end{minipage} & \begin{minipage}[b]{\linewidth}\raggedright
OpenCode\\
reference\strut
\end{minipage} \\
\midrule\noalign{}
\endfirsthead
\toprule\noalign{}
\begin{minipage}[b]{\linewidth}\raggedright
Component
\end{minipage} & \begin{minipage}[b]{\linewidth}\raggedright
Codex\\
raw\strut
\end{minipage} & \begin{minipage}[b]{\linewidth}\raggedright
Codex\\
+ Cortex\strut
\end{minipage} & \begin{minipage}[b]{\linewidth}\raggedright
Claude\\
raw\strut
\end{minipage} & \begin{minipage}[b]{\linewidth}\raggedright
Claude\\
+ Cortex\strut
\end{minipage} & \begin{minipage}[b]{\linewidth}\raggedright
OpenCode\\
reference\strut
\end{minipage} \\
\midrule\noalign{}
\endhead
\bottomrule\noalign{}
\endlastfoot
Functional & 0.5455 & 0.4545 & 0.4545 & 0.4545 & 0.7273 \\
VERTEX & 0.5214 & 0.4776 & 0.4512 & 0.4199 & 0.5070 \\
PWA / accessibility & 1.0000 & 0.8750 & 0.7500 & 0.6250 & 0.6250 \\
Visual similarity & 0.6154 & 0.6138 & 0.6509 & 0.6242 & 0.6605 \\
Code / architecture & 0.7328 & 0.7328 & 0.7563 & 0.7519 & 0.6591 \\
Robustness & 1.0000 & 1.0000 & 1.0000 & 1.0000 & 1.0000 \\
Code health & 1.0000 & 1.0000 & 1.0000 & 1.0000 & 0.6000 \\
Static security & 1.0000 & 1.0000 & 1.0000 & 1.0000 & 1.0000 \\
Composite & 0.7212 & 0.6799 & 0.6737 & 0.6533 & 0.7071 \\
\end{longtable}

\end{minipage}\par\medskip

Figure 8 gives a common-scale view of all these components. The composite uses authored-reference VERTEX; the estimated-reference diagnostic in Appendix B.5 does not replace that component. The pools do not isolate governance: configuration, repair opportunities, and final-evaluation independence were not held fixed.

\begin{figure}
\centering
\includegraphics[width=1\linewidth,height=\textheight,keepaspectratio]{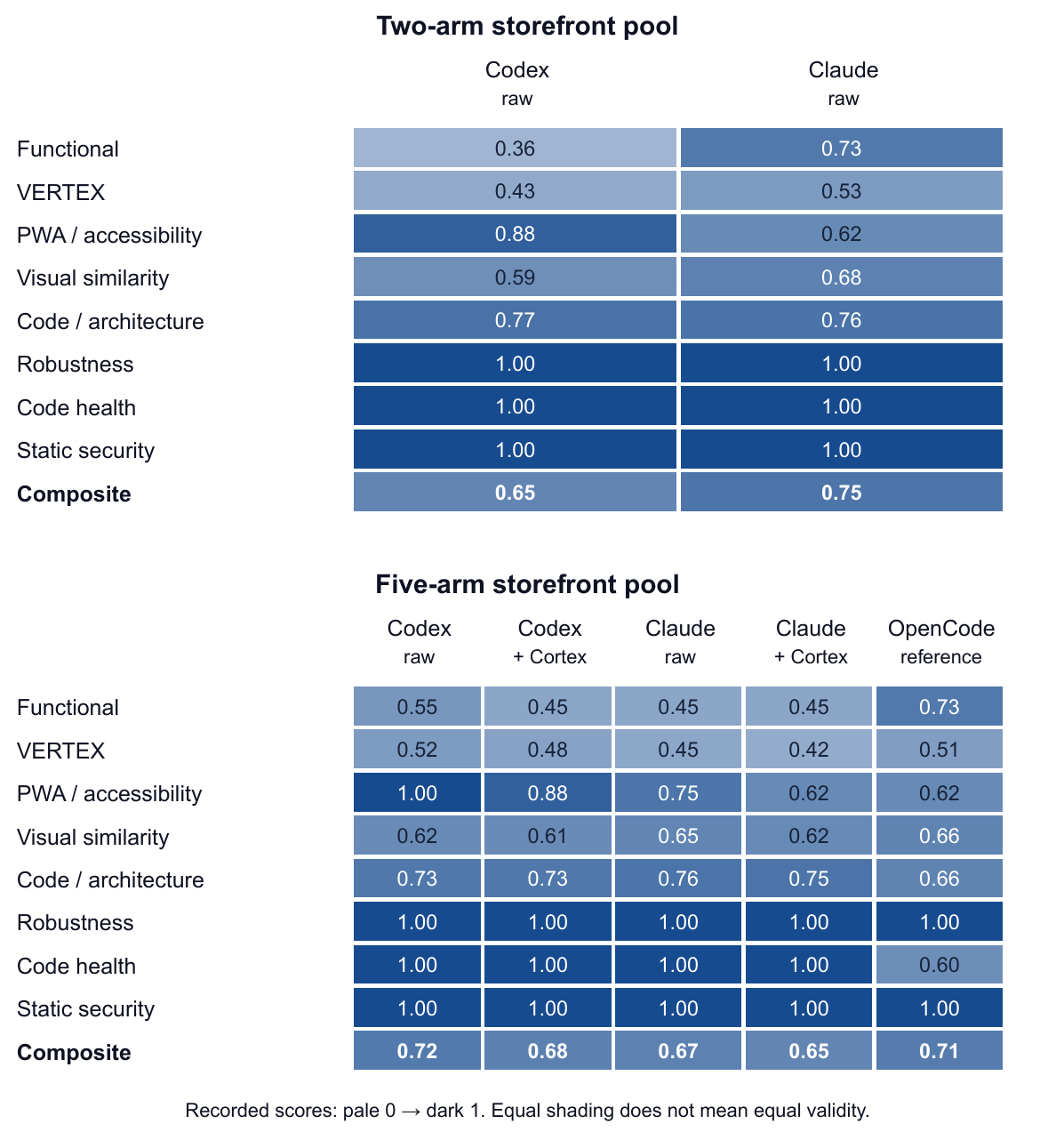}
\caption{Recorded storefront score profiles, rounded to two decimals for display. Tables 7--8 retain four-decimal values and distinguish the two separately captured candidate pools. Component scores are proxies, not success probabilities; perfect recorded robustness or static-security values do not establish exhaustive testing or absence of vulnerabilities. The composite is a weighted, build-gated combination, not an additional independent measurement.}\label{fig-8}
\end{figure}

\subsection{B.5 Authored and estimated reference ranks}\label{b.5-authored-and-estimated-reference-ranks}

The authored reference contains 13 capability descriptions and five architectural anchors. Estimated references use the public brief and the synthetic design vocabulary explained in Section 5.1.1. The pool diagnostic can also use descriptors supported by at least two other arms, excluding the candidate itself. This peer component is unavailable in the two-arm pool and can contribute in the five-arm pool, so reference construction differs.

Table 9 contains all score pairs underlying Figure 6. Both methods rank Claude above Codex in the two-arm pool; with two unequal observations, matching order alone yields correlation \(1\). The five-arm pool exchanges raw Codex and Cortex over Codex between ranks one and three, leaving the remaining ranks unchanged and yielding correlation \(0.6\).

\begin{minipage}{\linewidth}

\begin{longtable}[]{@{}
  >{\raggedright\arraybackslash}p{(\linewidth - 8\tabcolsep) * \real{0.1600}}
  >{\raggedright\arraybackslash}p{(\linewidth - 8\tabcolsep) * \real{0.2800}}
  >{\raggedright\arraybackslash}p{(\linewidth - 8\tabcolsep) * \real{0.2000}}
  >{\raggedright\arraybackslash}p{(\linewidth - 8\tabcolsep) * \real{0.2000}}
  >{\raggedright\arraybackslash}p{(\linewidth - 8\tabcolsep) * \real{0.1600}}@{}}
\caption{Exact VERTEX pairs and within-pool ranks (1 = highest) for Figure 6. Rank arrows mean authored to estimated rank, not improvement over time. The two pools share one storefront brief; seven pairs are not seven independent task samples. Peer-consensus descriptors are unavailable with only two arms.}\label{tbl-9}\tabularnewline
\toprule\noalign{}
\begin{minipage}[b]{\linewidth}\raggedright
Candidate pool
\end{minipage} & \begin{minipage}[b]{\linewidth}\raggedright
Configuration
\end{minipage} & \begin{minipage}[b]{\linewidth}\raggedright
Authored VERTEX
\end{minipage} & \begin{minipage}[b]{\linewidth}\raggedright
Estimated VERTEX
\end{minipage} & \begin{minipage}[b]{\linewidth}\raggedright
Rank change
\end{minipage} \\
\midrule\noalign{}
\endfirsthead
\toprule\noalign{}
\begin{minipage}[b]{\linewidth}\raggedright
Candidate pool
\end{minipage} & \begin{minipage}[b]{\linewidth}\raggedright
Configuration
\end{minipage} & \begin{minipage}[b]{\linewidth}\raggedright
Authored VERTEX
\end{minipage} & \begin{minipage}[b]{\linewidth}\raggedright
Estimated VERTEX
\end{minipage} & \begin{minipage}[b]{\linewidth}\raggedright
Rank change
\end{minipage} \\
\midrule\noalign{}
\endhead
\bottomrule\noalign{}
\endlastfoot
Two-arm & Codex raw & 0.4309 & 0.4142 & 2 → 2 \\
Two-arm & Claude raw & 0.5271 & 0.5465 & 1 → 1 \\
Five-arm & Codex raw & 0.5214 & 0.6026 & 1 → 3 \\
Five-arm & Cortex over Codex & 0.4776 & 0.6465 & 3 → 1 \\
Five-arm & Claude raw & 0.4512 & 0.5971 & 4 → 4 \\
Five-arm & Cortex over Claude & 0.4199 & 0.5126 & 5 → 5 \\
Five-arm & OpenCode & 0.5070 & 0.6159 & 2 → 2 \\
\end{longtable}

\end{minipage}\par\medskip

Both diagnostics concern one independently specified brief. No cross-brief holdout evaluation or correlation interval is reported. Similarity-rank agreement is not agreement on complete task correctness; peer-derived references can reproduce shared errors. These observations motivate further validation, not discarding authored references or claiming measured long-horizon performance.

\end{document}